\documentclass{article}

\usepackage{microtype}
\usepackage{graphicx}
\usepackage{enumitem}
\usepackage{booktabs} %

\usepackage{amsmath,amsfonts,bm}

\def\eqref#1{equation~\ref{#1}}

\def\1{\bm{1}}

\DeclareMathAlphabet{\mathsfit}{\encodingdefault}{\sfdefault}{m}{sl}
\SetMathAlphabet{\mathsfit}{bold}{\encodingdefault}{\sfdefault}{bx}{n}

\usepackage{hyperref}

\usepackage[accepted]{icml2024}

\usepackage{amsmath}
\usepackage{amssymb}
\usepackage{mathtools}
\usepackage{amsthm}

\usepackage{hyperref}
\usepackage{url}
\usepackage{wrapfig}
\usepackage{threeparttable}
\usepackage{subfig}
\usepackage{wrapfig}
\usepackage{algorithm}
\usepackage{algorithmic}
\usepackage{bbding}
\usepackage{utfsym}
\usepackage{makecell}
\usepackage{multirow}

\usepackage[capitalize,noabbrev]{cleveref}

\theoremstyle{plain}
\newtheorem{theorem}{Theorem}[section]
\newtheorem{proposition}[theorem]{Proposition}

\newtheorem{observation}{Observation}

\theoremstyle{definition}
\newtheorem{definition}[theorem]{Definition}

\theoremstyle{remark}

\usepackage[textsize=tiny]{todonotes}

\icmltitlerunning{Enhancing Adversarial Robustness in SNNs with Sparse Gradients}

\begin{document}

\twocolumn[
\icmltitle{Enhancing Adversarial Robustness in SNNs with Sparse Gradients}

\icmlsetsymbol{equal}{*}

\begin{icmlauthorlist}
\icmlauthor{Yujia Liu}{nersvt,lab,cs}
\icmlauthor{Tong Bu}{nersvt,lab,ai}
\icmlauthor{Jianhao Ding}{nersvt,lab,cs}
\icmlauthor{Zecheng Hao}{nersvt,lab,cs}
\icmlauthor{Tiejun Huang}{nersvt,lab,cs}
\icmlauthor{Zhaofei Yu}{nersvt,lab,ai}
\end{icmlauthorlist}

\icmlaffiliation{nersvt}{NERCVT, School of Computer Science, Peking University, China}
\icmlaffiliation{lab}{National Key Laboratory for Multimedia Information Processing, Peking University, China}
\icmlaffiliation{cs}{School of Computer Science, Peking University, China}
\icmlaffiliation{ai}{Institution for Artificial Intelligence, Peking University, China}

\icmlcorrespondingauthor{Zhaofei Yu}{yuzf12@pku.edu.cn}

\icmlkeywords{Machine Learning, ICML}

\vskip 0.3in
]

\printAffiliationsAndNotice{}  %

\begin{abstract}
Spiking Neural Networks (SNNs) have attracted great attention for their energy-efficient operations and biologically inspired structures, offering potential advantages over Artificial Neural Networks (ANNs) in terms of energy efficiency and interpretability. Nonetheless, similar to ANNs, the robustness of SNNs remains a challenge, especially when facing adversarial attacks. Existing techniques, whether adapted from ANNs or specifically designed for SNNs, exhibit limitations in training SNNs or defending against strong attacks.
In this paper, we propose a novel approach to enhance the robustness of SNNs through gradient sparsity regularization. We observe that SNNs exhibit greater resilience to random perturbations compared to adversarial perturbations, even at larger scales. Motivated by this, we aim to narrow the gap between SNNs under adversarial and random perturbations, thereby improving their overall robustness. 
To achieve this, we theoretically prove that this performance gap is upper bounded by the gradient sparsity of the probability associated with the true label
concerning the input image, laying the groundwork for a practical strategy to train robust SNNs by regularizing the gradient sparsity. 
We validate the effectiveness of our approach through extensive experiments on both image-based and event-based datasets.
 The results demonstrate notable improvements in the robustness of SNNs.
Our work highlights the importance of gradient sparsity in SNNs and its role in enhancing robustness.
\end{abstract}

\section{Introduction}

Although Artificial Neural Networks (ANNs) have achieved impressive performance across various tasks~\citep{2016_CVPR_He_ResNet,2021_NeuroComput_Mishra_Diagnosis,2021_CompSur_Khan_AutoDrive}, they are often plagued by complex computations and limited interpretability~\citep{2021_Liu_SwinTrans,2018_Zachary_Inter}.
In recent years, Spiking Neural Networks (SNNs) have garnered significant attention in the field of artificial intelligence due to their energy-efficient operations and biologically-inspired architectures~\citep{maas1997networks,zenke2021visualizing}. In SNNs, neurons simulate changes in membrane potentials and transmit information through spike trains~\citep{roy2019towards}. These characteristics allow for a certain level of biological interpretation while avoiding the extensive and complex matrix multiplication operations inherent in ANNs.
Remarkably, SNNs have achieved competitive performance with ANNs on various classification datasets~\citep{sengupta2019going,fang2021deep,deng2021temporal,Xu_2023_NIPS}.

Similar to ANNs~\citep{2016_arxiv_Tanay_LinearCase,2019_CVPR_Stutz_OffManifold,2022_CS_Liu_LPIPISAttack}, the issue of robustness poses a significant challenge for SNNs~\citep{sharmin2019comprehensive,sharmin2020inherent,kundu2021hire}. When subjected to imperceptible perturbations added to input images, SNNs can exhibit misclassifications, which are known as adversarial attacks. Developing techniques to train adversarially robust SNNs remains an ongoing problem in the research community.
Some successful techniques designed for ANNs have been adapted for use with SNNs, including adversarial training~\citep{aleksander2018towards,2022_AI_Ho_AdvTrain,ding2022snn} and certified training~\citep{2020_ICLR_Zhang_IBP,2021_ICML_Zhang_LinfNet,liang2022toward}. 
Additionally, there are SNN-specific methods proposed to improve robustness, such as temporal penalty configurations~\citep{leontev2021robustness}, and specialized coding schemes~\citep{sharmin2020inherent}.
However, these methods either prove challenging to train on SNNs~\citep{liang2022toward} or exhibit limited effectiveness against strong attacks~\citep{sharmin2020inherent}.

In this paper, we present a novel approach to enhance the robustness of SNNs by considering the gradient sparsity with respect to the input image.
We find that SNNs exhibit greater robustness to random perturbations, even at larger scales, compared to adversarial perturbations.
Building upon this observation, we propose to minimize the performance gap between an SNN subjected to adversarial perturbations and random perturbations, thereby enhancing its overall robustness. The main contributions of our work are as follows and the code of this work is accessible at \url{https://github.com/putshua/gradient_reg_defense}.
\begin{itemize}
\setlength{\itemsep}{0pt}
\setlength{\parsep}{0pt} 
\setlength{\parskip}{0pt}
    \item We analyze the robustness of SNNs and reveal that SNNs exhibit robustness against random perturbations even at significant scales, but display vulnerability to small-scale adversarial perturbations.
    \item We provide theoretical proof that the gap between the robustness of SNNs under these two types of perturbations is upper bounded by the sparsity of gradients of the probability associated with the true label with respect to the input image.
    \item We propose to incorporate gradient sparsity regularization into the loss function during training to narrow the gap, thereby boosting the robustness of SNNs.
    \item Extensive experiments on the image-based and event-based datasets validate the effectiveness of our method, which significantly improves the robustness of SNNs.
\end{itemize}

\section{Related Work}
\subsection{Learning Algorithms of SNNs}
The primary objective of most SNN learning algorithms is to achieve high-performance SNNs with low latency. Currently, the most effective and popular learning algorithms for SNNs are the ANN-SNN conversion~\citep{cao2015spiking} and supervised learning~\citep{wu2018STBP}.
The ANN-SNN conversion method aims to obtain SNN weights from pre-trained ANNs with the same network structure. By utilizing weight scaling~\citep{li2021free, hu2023spiking}, threshold balancing~\citep{diehl2015fast, deng2020optimal}, quantization training techniques~\citep{bu2021optimal}, and spike calibration~\citep{hao2023bridging}, well-designed ANN-SNN algorithms can achieve lossless performance compared to the original ANN~\citep{Han_2020_CVPR, ho2021tcl}. However, the converted SNNs often require larger timesteps to achieve high performance, resulting in increased energy consumption. Moreover, they lose temporal information and struggle to process neuromorphic datasets.
The supervised learning approach directly employs the backpropagation algorithm to train SNNs with fixed timesteps. Wu et al.,~\yrcite{wu2018STBP, wu2019direct} borrowed the idea from the Back Propagation Through Time (BPTT) in RNN learning and proposed the Spatio-Temporal-Back-Propagation (STBP) algorithm. They approximate the gradient of spiking neurons using surrogate functions~\citep{neftci2019surrogate}. While supervised training significantly improves the performance of SNNs on classification tasks~\citep{ kim2020revisiting, lee2020enabling, fang2021incorporating,  zheng2021going, guo2022loss, yao2022glif, duantemporal, mostafa2017supervised, bohte2000spikeprop, zhang2020temporal, zhang2022rectified, Xu_2023_TNNLS, zhutraining, zhu2024exploring}, SNNs still fall behind ANNs in terms of generalization and flexibility. Challenges such as gradient explosion and vanishing persist in SNNs.

\subsection{Defense Methods of SNNs}
Methods for improving the robustness of SNNs can be broadly categorized into two classes. The first class draws inspiration from ANNs. A typical representative is adversarial training, which augments the training set with adversarial examples generated by attacks~\citep{aleksander2018towards,tramer2018ensemble,wong2019fast}. This approach has been shown to effectively defend against attacks that are used in the training phase. Another method is certified training, which utilizes certified defense methods to train a network~\citep{2018_ICML_Wong_DefenseConvex,2020_NIPS_Xu_Certified}. Certified training has demonstrated promising improvements in the robustness of ANNs~\citep{2020_ICLR_Zhang_IBP}, but its application to SNNs remains challenging~\citep{liang2022toward}. 
The second category consists of SNN-specific techniques designed to enhance robustness. On one hand, the choice of encoding the continuous intensity of an image into 0-1 spikes can impact the robustness of SNNs. Recent studies have highlighted the Poisson encoder as a more robust option~\citep{sharmin2020inherent,kim2022rate}. 
However, the Poisson encoder generally yields worse accuracy on clean images than the direct encoding, and the robustness improvement caused by the Poisson encoder varies with the number of timesteps used.
On the other hand, researchers have recognized the unique temporal dimension of SNNs and developed strategies related to temporal aspects to improve robustness~\cite{nomura2022robustness}. Hao et al.~\yrcite{haothreaten} further pointed out the significance of utilizing the rate and temporal information comprehensively to enhance the reliability of SNNs.
Apart from the studies on static datasets, there are works that attempt to perform adversarial attacks and defenses on the Dynamic Vision Sensors (DVS) dataset~\citep{marchisio2021r}.
In this paper, we mainly focus on the direct encoding of input images. We propose a gradient sparsity regularization strategy to improve SNNs' robustness with theoretical guarantees. Moreover, this strategy can be combined with adversarial training to further boost the robustness of SNNs.

\section{Preliminary}
\subsection{Neuron Dynamics in SNNs}
Similar to previous works~\citep{wu2018STBP,rathi2021diet}, we consider the commonly used Leaky Integrate-and-Fire (LIF) neuron model due to its efficiency and simplicity, the dynamic of which can be formulated as follows:
\begin{align}
\label{eq:LIF-model}
    u_i^l[t] &= \tau u_i^l[t-1](1-s_i^l[t-1]) + \sum_{j} w_{i,j}^{l-1} s_j^{l-1}[t], \\
\label{eq:Heaviside}
    s_i^l[t] & = H(u_i^l[t]-\theta).
\end{align}
 Equation~(\ref{eq:LIF-model}) describes the membrane potential of the $i$-th neuron in layer $l$, which receives the synaptic current from the $j$-th neuron in layer $l-1$. Here $\tau$ represents the membrane time constant and $t$ denotes the discrete time step ranging from $1$ to $T$.  The variable $u_i^l[t]$ represents the membrane potential of the $i$-th neuron in layer $l$ at the time step $t$. 
$w_{i,j}$ denotes the synaptic weight between the two neurons, and $s_j^{l-1}[t]$ represents the binary output spike of neuron $j$ in layer $l-1$. For simplicity, the resting potential is assumed to be zero so that the membrane potential will be reset to zero after firing.
Equation~(\ref{eq:Heaviside}) defines the neuron fire function. At each time step $t$, a spike will be emitted when the membrane potential $u_i^{l}[t]$ surpasses a specific threshold $\theta$. The function $H(\cdot)$ denotes the Heaviside step function, which equals 0 for negative input and 1 for others.

\subsection{Adversarial Attacks for SNNs}
Preliminary explorations have revealed that SNNs are also susceptible to adversarial attacks~\citep{sharmin2020inherent, kundu2021hire,liang2021exploring,marchisio2021dvs}. Well-established techniques such as the Fast Gradient Sign Method (FGSM) and Projected Gradient Descent (PGD) can generate strong adversaries that threaten SNNs. 

\textbf{FGSM}~\citep{goodfellow2014explaining} is a straightforward none-iterative attack method, expressed as follows:
\begin{align}
\label{eq:FGSM}
    \hat{\bm{x}} = \bm{x} + \epsilon~\text{sign} (\nabla_{\bm{x}} \mathcal{L} (f(\bm{x}),y)),
\end{align}
where $\bm{x}$ and $\hat{\bm{x}}$ represent the original image and the adversarial example respectively, $\epsilon$ denotes the perturbation bound, $\mathcal{L}$ refers to the loss function, $f(\cdot)$ represents the neural network function, and $y$ denotes the label data.

\textbf{PGD}~\citep{aleksander2018towards} is an iterative extension of FGSM, which can be described as follows:
\begin{equation}
\label{eq:PGD}
    \hat{\bm{x}}^k = \Pi_\epsilon \{ \bm{x}^{k-1} + \alpha~\text{sign} (\nabla_{\bm{x}} \mathcal{L} (f(\bm{x}^{k-1}),y)) \},
\end{equation}
where $k$ is the current iteration step and $\alpha$ is the step size. 
The operator $\Pi_\epsilon$ projects the adversarial examples onto the space of the $\epsilon$ neighborhood in the $\ell_\infty$ norm around $\bm{x}$.

\section{Methodology}
In this section, we first compare the vulnerability of SNNs to random perturbations versus adversarial perturbations. We highlight that SNNs exhibit significant robustness against random perturbations but are more susceptible to adversarial perturbations. Then we quantify the disparity between adversarial vulnerability and random vulnerability, proving that it is upper bounded by the gradient sparsity of the probability related to the true label concerning the input image. Based on this, we propose a novel approach to enhance the robustness of SNNs by introducing Sparsity Regularization (SR) of gradients in the training phase and incorporating this regularization into the learning rule of SNNs.

\begin{figure}[t]
    \centering
    \includegraphics[width=0.48\textwidth]{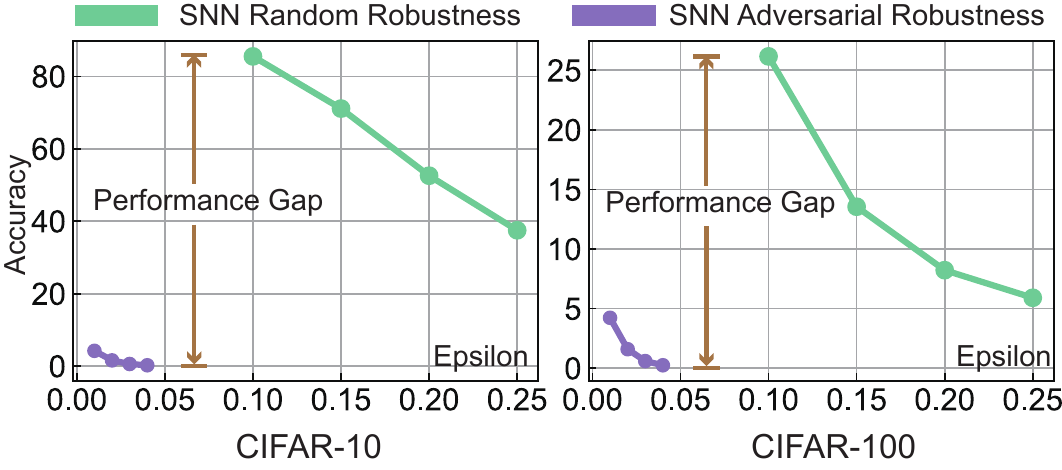}
    \caption{Comparison of the random vulnerability and adversarial vulnerability of SNNs on CIFAR-10 and CIFAR-100.}
    \label{fig:rand_vs_adv} 
\end{figure}

\subsection{Compare Random and Adversarial Vulnerability}
The objective of adversarial attacks is to deliberately alter the probability vector $f(\bm{x})$ in order to change the classification result.
Since the classification result is determined by the magnitude-based ordering of components in $f(\bm{x})$,
attackers aim to substantially decrease the value of the component corresponding to the true label of $\bm{x}$.
We suppose $\bm{x}$ with label $y$ belonging to the $y$-th class, so that the value of $f_{y}(\bm{x})$ has a substantial impact on the classification result.
Therefore, the stability of $f_{y}(\bm{x})$ becomes crucial for the robustness of SNNs, particularly in terms of the value of $|f_{y}(\hat{\bm{x}}) - f_y(\bm{x})|$ when subjected to small perturbation on $\bm{x}$.

To initiate our analysis, we define the random and adversarial vulnerability of an SNN, denoted as $f$, at a specific point $\bm{x}$ under an $\ell_p$ attack of size $\epsilon$.
\begin{definition}(Random Vulnerability)
	The random vulnerability of $f$ at point $\bm{x}$ to an $\ell_p$ attack of size $\epsilon$ is defined as the expected value of $( f_y(\bm{x}+\epsilon\cdot\delta) - f_y(\bm{x}) )^2$, where $\delta$ follows a uniform distribution within the unit $\ell_p$ ball, and $y$ represents the class of $\bm{x}$ belonging to. Mathematically, it can be expressed as:
	\begin{equation}
      \label{random}
		\rho_\text{rand}(f, \bm{x}, \epsilon, \ell_p) = 
		\mathop{\mathbb{E}}\limits_{\delta \sim U\{
                    \Vert \delta \Vert_p \leqslant  1
                    \}} \left( f_y(\bm{x}+\epsilon\cdot\delta) - f_y(\bm{x}) \right)^2.
	\end{equation}
\end{definition}

\begin{definition}(Adversarial Vulnerability)
	The adversarial vulnerability of $f$ at point $\bm{x}$ to an $\ell_p$ attack of size $\epsilon$ is defined as the supremum of $( f_y(\bm{x}+\epsilon\cdot\delta) - f_y(\bm{x}) )^2$, where $\delta$ follows a uniform distribution within the unit $\ell_p$ ball, and $y$ represents the class of $\bm{x}$ belonging to.
Mathematically, it can be expressed as:
	\begin{equation}
        \label{adversarial}
		\rho_\text{adv}(f, \bm{x},\epsilon, \ell_p) = 
		\sup_{\delta \sim U\{
                    \Vert \delta \Vert_p  \leqslant  1
                    \}} \left( f_y(\bm{x}+\epsilon\cdot\delta) - f_y(\bm{x}) \right)^2.
	\end{equation}
\end{definition}
\begin{figure*}[t]  
    \centering
    \includegraphics[width=\linewidth]{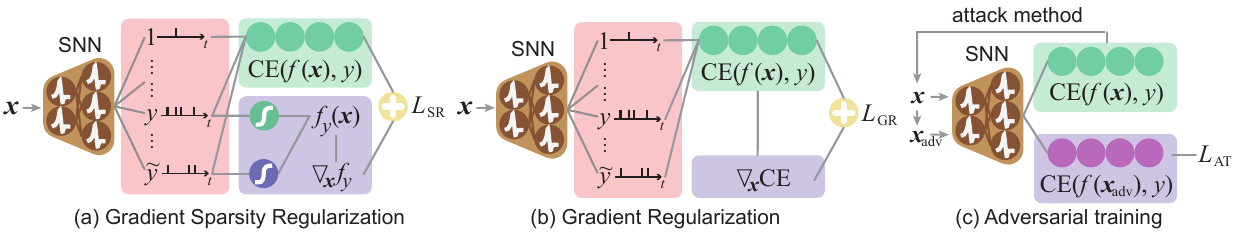} \\
    \caption{Illustration of (a) the proposed SR strategy, (b) gradient regularization and (c) adversarial training.}
    \label{fig:algo}
\end{figure*}
To gain a deeper understanding of the disparity in vulnerability between random perturbations and adversarial perturbations in SNNs, we conduct a small-scale experiment with a primary focus on $\ell_\infty$ attacks. The results of the experiment are depicted in Figure~\ref{fig:rand_vs_adv}.
We specifically prioritized $\ell_\infty$ attacks over other $\ell_p$ attacks due to the widespread utilization of $\ell_\infty$ constraints in various attack methods.
Adversarial examples generated under $\ell_\infty$ attacks tend to be more destructive compared to those generated under $\ell_0$ and $\ell_2$ attacks~\cite{aleksander2018towards}. 
Therefore, focusing on $\ell_\infty$ attacks allows us to assess the robustness of the SNN against the most severe random and adversarial perturbations.

In our experiment, we evaluate the performance of a well-trained SNN $f$ with the VGG-11 architecture. 
Our primary objective is to examine the impact of both random and adversarial perturbations on the SNN's classification results. For a given image $\bm{x}$, we added random perturbations uniformly drawn from a hyper-cube $\{ \delta_\text{rand}:\Vert \delta_\text{rand} \Vert_\infty \leqslant \epsilon\}$ to the original image, resulting in the perturbed image $\hat{\bm{x}}_\text{rand}$. Additionally, we employed a $\epsilon$-sized FGSM attack to generate an adversarial example $\hat{\bm{x}}_\text{adv}$ from $\bm{x}$.
Subsequently, we individually input both $\hat{\bm{x}}_\text{rand}$ and $\hat{\bm{x}}_\text{adv}$ into the network $f$ to observe any changes in the classification results.
We evaluate the classification accuracy of the perturbed images on the test set of CIFAR-10 and CIFAR-100 under both random and adversarial perturbations, respectively. The results, as depicted in Figure~\ref{fig:rand_vs_adv}, yield the following key findings:

\begin{observation}
\label{observ:rand_adv}
SNNs exhibit robustness against random perturbations even when the perturbation scale is significant, but display vulnerability to small-scale adversarial perturbations.
\end{observation}

\subsection{Quantify the Gap between Random Vulnerability and Adversarial Vulnerability}

Observation~\ref{observ:rand_adv} indicates that SNNs exhibit notable robustness to random perturbations in comparison to adversarial perturbations. To enhance the adversarial robustness of SNNs, a natural approach is to minimize the disparity between adversarial vulnerability and random vulnerability.

To measure the disparity in vulnerability between the SNN $f$ under adversarial perturbations and random noise, we employ the ratio of $\rho_\text{adv}(f, \bm{x}, \epsilon, \ell_\infty)$ and $\rho_\text{rand}(f, \bm{x}, \epsilon, \ell_\infty)$.
We make the assumption that $\rho_\text{rand}(f, \bm{x}, \epsilon, \ell_\infty)\ne 0$, indicating that $f$ is not a constant function. This assumption aligns with the practical reality of SNNs in real-world applications. 
 Optimizing the ratio of $\rho_\text{adv}(f, \bm{x}, \epsilon, \ell_\infty)$ and $\rho_\text{rand}(f, \bm{x}, \epsilon, \ell_\infty)$ directly is a challenging task. However, we are fortunate to present a mathematical proof that establishes an upper bound for this ratio based on the sparsity of $\nabla_{\bm{x}} f_y$. Specifically, we have the following theorem.
\begin{theorem}
	\label{thr:binary_sparsity}
	Suppose $f$ is a differentiable SNN by surrogate gradients, and $\epsilon$ is the magnitude of an attack, assumed to be small enough. Given an input image $\bm{x}$ with corresponding label $y$, the ratio of $\rho_\text{adv}(f, \bm{x},\epsilon, \ell_\infty)$ and $ \rho_\text{rand}(f, \bm{x},\epsilon, \ell_\infty)$ is upper bounded by the sparsity of $\nabla_{\bm{x}} f_y$:
	\begin{equation}
		3 \leqslant 
                \frac{\rho_\text{adv}(f, \bm{x},\epsilon, \ell_\infty)}
                {\rho_\text{rand}(f, \bm{x},\epsilon, \ell_\infty)} 
                \leqslant 3\Vert \nabla_{\bm{x}} f_y(\bm{x}) \Vert_0.
	\end{equation}
\end{theorem}
The proof is provided in Appendix~\ref{append:proof_theorem}.
This theorem illustrates that the disparity between the adversarial vulnerability and random vulnerability is upper bounded by the sparsity of $\nabla_{\bm{x}} f_y$.
It provides valuable insights into the correlation between gradient sparsity and the disparity in robustness exhibited by SNNs when subjected to different perturbations with $\ell_\infty$ attacks.
According to Theorem~\ref{thr:binary_sparsity}, we can infer that a sparser gradient contributes to closing the robustness gap between SNN $f$ under worst-case scenarios and its robustness under random perturbations.

From an intuitive perspective, minimizing the $\ell_0$ norm of $\nabla_{\bm{x}} f_y(\bm{x})$
serves to bring $\bm{x}$ closer to a local minimum point.
In an ideal scenario, this would entail trapping $\bm{x}$ within a local minimum, effectively rendering attackers unable to generate adversarial examples through gradient-based methods.
By introducing the sparsity constraint for each $\bm{x}$ in the training set, we encourage learning an SNN $f$ where input images tend to remain close to extreme points or trapped in local minimums.
This makes it challenging to perturb $f_y(\bm{x})$ with small perturbations, thereby enhancing the robustness of the SNN $f$.

\subsection{Loss Function with Sparsity Regularization}
To promote sparsity of the gradients, a straightforward approach is to incorporate the $\Vert \nabla_{\bm{x}} f_y(\bm{x}) \Vert_0$ term into the training loss, where $f_y(\bm{x})$ is the probability assigned by $f$ to $\bm{x}$ belonging to the true label $y$. 

Consider an SNN $f$ with a final layer denoted as $L$. The total number of neurons in layer $L$ is denoted by $N$, and the time-step $t$ ranges from $1$ to $T$.
For a given input image $\bm{x}$, the output vector of the $L^{th}$ layer depends on the collective outputs across all time-steps: 
\begin{equation}
    f^L(\bm{x}) = \left(
        \sum_{t=1}^T s_1^L(t),\dots,\sum_{t=1}^T s_N^L(t)
        \right)^T.
\end{equation}
While regularizing $\Vert \nabla_{\bm{x}} f^L_y(\bm{x}) \Vert_0$ is a straightforward way to keep the stability of $f^L_y(\bm{x})$, it may be insufficient for multi-classification tasks.
This is because the classification result is influenced not only by the value of $f_y^L$, but also by the magnitude of $f_y^L$ in comparison to the other components of $f_y^L$.
To account for such relationships while maintaining computational efficiency, we utilize two spike streams at the last layer of $f$ to calculate $f_y(\bm{x})$, as illustrated in Figure~\ref{fig:algo} (a). This is expressed as:
\begin{equation}
\begin{split}
    f_y & = \frac{
    e^{\sum_{t=1}^T s_y^L(t)}
    }{e^{\sum_{t=1}^T s_y^L(t)} + e^{\sum_{t=1}^T s_{\tilde{y}}^L(t)}}, \\
    f_{\tilde{y}} &= \frac{
    e^{\sum_{t=1}^T s_{\tilde{y}}^L(t)}
    }{e^{\sum_{t=1}^T s_y^L(t)} + e^{\sum_{t=1}^T s_{\tilde{y}}^L(t)}}.
\end{split}
\label{eq:f_y}
\end{equation}
where $\tilde{y}$ represents the index of the maximum component in $\{f_i^L(\bm{x}): i\ne y\}$. 

On one hand, the transformation in Equation~(\ref{eq:f_y}) introduces only one additional component in $f^L$ while preserving the classification results derived from $f^L$, as expressed by
\begin{equation}
    \mathop{\arg\max}_{i=1,\dots,N} f_i^L(\bm{x}) = \mathop{\arg\max}_{i=y,\tilde{y}} f_i (\bm{x}).
\end{equation}
On the other hand, since $f_y + f_{\tilde{y}} = 1$, so if $f_y$ is stable, $f_{\tilde{y}}$ is also stable. Thus, it is sufficient to regularize $\Vert \nabla_{\bm{x}} f_y (\bm{x}) \Vert_0$ to enhance the adversarial robustness of the SNN $f$.

Therefore, the training loss can be written as:
\begin{equation}
    \mathcal{L}(\bm{x},y) = \text{CE}(f^L(\bm{x}), y) + \lambda \Vert \nabla_{\bm{x}} f_y(\bm{x}) \Vert_0,
    \label{eq:08}
\end{equation}
where $\text{CE}(\cdot)$ is the cross-entropy loss, $f^L(\bm{x})$ is the output of the last layer,
and $\lambda$ denotes the coefficient parameter controlling the strength of the sparsity regularization.

Here we give the formulation of $\nabla_{\bm{x}} f_y(\bm{x})$ and the detailed derivation is provided in Appendix~\ref{sec:der_eq10}.
Given an input image $\bm{x}$ belonging to the class $y$, $\nabla_{\bm{x}} f_y(\bm{x})$ can be formulated as:
\begin{equation}
\label{eq:fy_gradient}
    \nabla_{\bm{x}} f_y(\bm{x}) = \sum_{i=y,\tilde{y}} \left( \frac{\partial f_y
                    }{
                        \partial f_i^L
                    }
                    \left(
                        \sum_{t=1}^T \sum_{\tilde{t}=1}^t \nabla_{\bm{x}[\tilde{t}]} s_i^L[t]
                    \right) \right).
\end{equation}
It is worth noting that the optimization problem involving the $\ell_0$ norm is known to be NP-hard~\citep{1995_SIAM_natarajan_L0Linear}.
To circumvent this challenge, we employ the $\ell_1$ norm as a substitute for the $\ell_0$ norm because the $\ell_1$ norm serves as a convex approximation to the $\ell_0$ norm~\citep{2013_Ramirez_L1ApproxL0}.
However, the computational burden associated with calculating the back-propagation of $\Vert \nabla_{\bm{x}} f_y(\bm{x}) \Vert_1$ is significant. Meanwhile, we find that SNNs are not trainable with such a double backpropagation approach.
Therefore, we adopt a finite difference approximation for this term~\citep{2021_Chris_GradientApproximation}.
In our case, we approximate the gradient regularization term using the following finite differences:
\begin{proposition}
\label{prop:L1approx}
    Let $\bm{d}$ denote the signed input gradient direction: $\bm{d}=sign(\nabla_{\bm{x}} f_y(\bm{x}))$, and $h$ be the finite difference step size. Then, the $\ell_1$ gradient norm can be approximated as:
    \begin{equation}
        \Vert \nabla_{\bm{x}} f_y(\bm{x}) \Vert_1 \approx \left| \frac{
        f_y(\bm{x}+h\cdot \bm{d}) - f_y(\bm{x})
        }{
        h
        } \right|.
    \end{equation}
\end{proposition}
The proof is provided in the Appendix~\ref{append:proof_prop}.
Finally, the training loss (Equation~(\ref{eq:08})) is rewritten as:
\begin{equation}
\label{eq:loss_sr}
    \mathcal{L}(\bm{x},y) = \text{ CE}(f^L(\bm{x}), y) + \lambda \left|
            \frac{f_y(\bm{x}+h\cdot \bm{d}) - f_y(\bm{x})
            }{
            h
            } \right|.
\end{equation}

The overall training algorithm is presented as Algorithm~\ref{alg:training}.

\begin{algorithm}[t]
	\caption{Training Algorithm}
    \textbf{Input}: Spiking neural network $f(\bm{x}, w)$ with parameter $w$\; Learning rate $\eta$; Step size $h$ of finite differences; Balance weight $\lambda$\\
    \textbf{Output}: Regularize trained parameter $w$
    
	\begin{algorithmic}[1]
    \label{alg:training}
	    \FOR{epoch=0 \textbf{to} n}
	        \STATE Sample minibatch $\{(\bm{x}^i,y^i)\}_{i=1,...,m}$ from Dataset %
            \FOR{$i=0$ \textbf{to} $m$}
                \STATE $f^L = f(\bm{x}^i, w)$
                \STATE $\tilde{y}^i = \arg\max_{j\ne y} f^L_j$
                \STATE $f_{y^i} = e^{f_{y^i}^L} / (e^{f_{y^i}^L} + e^{f_{\tilde{y}^i}^L})$
                \STATE $\bm{d}^i = \text{sign}( \nabla_x f_{y^i} ) \leftarrow$ the difference direction
                \STATE $\hat{\bm{x}}^i = \bm{x}^i+h \bm{d}^i$
                \STATE $\mathcal{L}(\bm{x}^i, y^i,w)$=$\text{CE}(f^L, y^i)$ 
                    + $\frac{\lambda}{h}|f_{y^i}(\hat{\bm{x}}^i) - f_{y^i}(\bm{x}^i)|$
                \STATE $w \leftarrow w - \eta \nabla_w \mathcal{L}(\bm{x}^i,y^i,w)$
            \ENDFOR
	    \ENDFOR
	\end{algorithmic}
\end{algorithm}

\begin{table*}[!t]
\caption{Comparison with the SOTA methods on classification accuracy (\%) under attacks.}
\centering
\resizebox{1\linewidth}{!}{
\begin{tabular}{cccccccccccc}
\toprule
\multirow{2}{*}{Dataset}
& \multirow{2}{*}{Arch.}
& \multirow{2}{*}{Defense}
& \multirow{2}{*}{Clean}
& \multicolumn{4}{c}{White Box Attack} 
& \multicolumn{4}{c}{Black Box Attack} \\
\cmidrule(lr){5-8} \cmidrule(lr){9-12}
&&&& PGD10& PGD30& PGD50& APGD10 & PGD10& PGD30& PGD50 & APGD10
\\
\midrule
 \multirow{3}{*}{CIFAR-10} & \multirow{3}{*}{VGG-11} &
 RAT & 90.44 & 11.53  & 7.08  & 6.41  & 3.26 & 43.29 & 40.17 & 40.16 & 47.50\\
 &&AT  & 89.97 & 18.18  & 14.79 & 14.63 & 10.36 & 44.02 & 43.38 & 43.40 & 52.90 \\
 &&$\text{SR}^*$ & 85.91 & 30.54  & 28.06 & 27.66 & 21.91 & 51.21 & 50.85 & 51.06 & 59.87\\ \midrule
 \multirow{3}{*}{CIFAR-10} & \multirow{3}{*}{WRN-16} &
 RAT & 92.70 & 10.52  & 5.14  & 4.33  & 2.19 & 38.35 & 31.04 & 30.40 & 36.42 \\
 &&AT  & 90.97 & 17.88  & 14.89 & 14.62 & 9.13 & 43.99 & 42.09 & 41.55 & 52.10 \\
 &&$\text{SR}^*$  & 85.63 & 39.18  & 37.04 & 36.74 & 29.03 & 50.84 & 50.23 & 49.97 & 57.93   \\ \midrule
 \multirow{3}{*}{CIFAR-100} & \multirow{3}{*}{WRN-16} &
 RAT & 69.10 & 5.72   & 3.58  & 3.26  & 2.08 & 22.61 & 18.77 & 18.26 & 25.23  \\
 &&AT  & 67.37 & 10.07  & 8.12  & 7.86  & 4.88 & 25.17 & 23.76 & 23.50 & 35.96  \\
 &&$\text{SR}^*$   & 60.37 & 19.76 & 18.39 & 18.11 & 13.32 & 28.38 & 28.01 & 27.94 & 36.69   \\
\bottomrule
\end{tabular}
}
\label{tab:sota}
\end{table*}

\subsection{Differences with Related Works}

We compare the proposed Sparsity Regularization (SR) strategy with Gradient Regularization (GR) as proposed by~\cite{2021_Chris_GradientApproximation} and classic adversarial training~\cite{aleksander2018towards} in Figure~\ref{fig:algo}. While GR relies on the selection of the multi-class calibrated loss function (e.g., cross-entropy or logistic loss), and adversarial training is associated with the attack method used in generating adversarial examples, the proposed SR strategy is both loss-independent and attack-independent.

\textbf{SR Strategy vs. GR Strategy.}. 
The main distinction between SR and GR strategies lies in that the regularization term in SR is exclusively tied to the \emph{model itself}, rather than being dependent on the multi-class calibrated loss used during the training phase, as shown in Figure~\ref{fig:algo} (b). To provide further clarification, we can express the training loss of the GR strategy as follows (defending against attacks):
\begin{equation}
    \ell(f^L(x),y) + \lambda \Vert \nabla_x \ell(f^L(x),y)\Vert_1^2
\end{equation}
where $\ell(f^L(x),y)$ is the multi-class calibrated loss function for the classification task (such as cross-entropy loss).
Note that the regularization term in GR varies depending on the specific choices of multi-class calibrated loss. In contrast, our proposed SR method introduces a regularization term that remains independent of the choice of loss function.
Moreover, our method can be combined with adversarial training to enhance the performance of adversarial training.

\textbf{SR Strategy vs. Adversarial Training.}
The key distinction between these two strategies lies in their approach. Adversarial training involves generating adversarial examples using specific attack methods, while the SR strategy is independent of adversarial examples. Figure~\ref{fig:algo} (c) illustrates that adversarial training aims to minimize the multi-class calibrated loss for adversarial images generated by attacks such as FGSM or PGD.
But the main idea of SR revolves around regularizing the sparsity of $\nabla_{\bm{x}} f_y(\bm{x})$. Although we employ the finite difference method (Equation~(\ref{eq:loss_sr})), to make the regularization of $\Vert \nabla_{\bm{x}} f_y(\bm{x}) \Vert_1$ computationally feasible, it is essential to recognize that the sample $\bm{x} + h\bm{d}$ fundamentally differs from an adversarial example in both numerical value and meaning.
On one hand, $\bm{x} + h\bm{d}$ is employed to calculate the difference quotient instead of directly calculating the multi-class calibrated loss. On the other hand, since $h$ serves as the step size for the approximation, it is advisable to use a small value to obtain a more accurate estimation of the $\ell_1$ norm. This contrasts with the requirement for a large $h$ when generating adversarial examples.

\section{Experiment}
\label{sec:experiment}
In this section, we evaluate the performance of the proposed SR strategy on image classification tasks using the CIFAR-10, CIRAR-100 and CIFAR10-DVS datasets. We adopt the experiment setting used in the previous work~\citep{ding2022snn}. Specifically, we use the VGG-11 architecture~\citep{simonyan2014very}, WideResNet with a depth of 16 and width of 4 (WRN-16)~\citep{zagoruyko2016wide}. The timestep for the SNNs is set to 8. Throughout the paper, we use the IF neuron with a hard-reset mechanism as the spiking neuron. 
Further details regarding the training settings can be found in the Appendix~\ref{append:train_settings}.

To generate adversarial examples, we employ different attacks, including FGSM~\citep{goodfellow2014explaining}, PGD~\citep{aleksander2018towards}, and AutoPGD~\citep{croce2020reliable}, with a fixed attack strength of $8/255$. For iterative attacks, the number of iterations is indicated in the attack name (e.g. PGD10). Since the choice of gradient approximation methods~\citep{bu2023rate} and surrogate functions can affect the attack success rate~\citep{xu2022securing}, we consider an ensemble attack for SNNs~\cite{Ozan2023adversarial}. We utilize a diverse set of surrogate gradients and consider both STBP-based~\cite{esser2016convolutional} and RGA-based~\cite{bu2023rate} attacks. For each test sample, we conduct multiple attacks using all possible combinations of gradient approximation methods and surrogate functions, and report the strongest attack. In other words, we consider an ensemble attack to be successful for a test sample as long as the model is fooled with any of the attacks from the ensemble.
Robustness is evaluated in two scenarios: the white-box scenario, where attackers have knowledge of the target model, and the black-box scenario, where the target model is unknown to attackers.
More detailed evaluation settings can be found in the Appendix~\ref{append:eval_settings}.
Moreover, for the ensemble attack to be meaningful and reveal any impact of gradient obfuscation, we run extensive experiment with varying widths of surrogate functions in Appendix~\ref{append:varying_width}.

\subsection{Compare with the State-of-the-art}
We validate the effectiveness of our method by comparing it with the current state-of-the-art approaches, including Regularized Adversarial Training (RAT)~\citep{ding2022snn} and Adversarial Training (AT)~\citep{kundu2021hire}. 
We use $\text{SR}^{*}$ to denote our sparsity regularization strategy with adversarial training.
For RAT-SNN, we replicate the model following the settings outlined in the paper~\citep{ding2022snn}. As for all SNNs trained with robustness training strategies, we adopt a PGD5 attack with $\epsilon=2/255$. 

Table~\ref{tab:sota} reports the classification accuracy of the compared methods under ensemble attacks. Columns 5-8 highlight the substantial enhancement in SNN robustness achieved through our strategy in the white box scenario.
Our proposed method consistently outperforms other State-Of-The-Art (SOTA) methods across all datasets and architectures.
For instance, when subjected to 10-steps PGD attacks, VGG-11 trained with our strategy elevates classification accuracy from 11.53\% (RAT) to an impressive 30.54\% on CIFAR-10.
Similarly, WRN-16 trained with SR* exhibits a remarkable 15\% boost in classification accuracy against PGD50 on CIFAR-100.
In comparison to the AT strategy, our method demonstrates a noteworthy enhancement in adversarial robustness, with a 10-20 percentage point improvement on both datasets under all attacks.

In contrast to white box attacks, all strategies exhibit better adversarial robustness against black box attacks (columns 9-12 in Table~\ref{tab:sota}).
When considering the CIFAR-10 dataset, models trained with any strategy achieve a classification accuracy of over 30\%  when subjected to PGD50.
However, models trained with the $\text{SR}^*$ strategy consistently outperform other strategies in all scenarios.
There exists a gap of 10\%-20\% in performance between models trained with RAT/AT and those trained with $\text{SR}^*$ on CIFAR-10, and a 5\%-10\% gap on CIFAR-100.
These results demonstrate the superiority of our approach over SOTA methods.

\subsection{Experiments on Dynamic Vision Sensor Data}

Since SNNs are suitable for application on neuromorphic data, we evaluate the effectiveness of the gradient sparsity regularization on CIFAR10-DVS dataset. Here, we use the preprocessed neuromorphic data and each data point contains ten frame-based data. The data batches are then fed into a 10-timestep spiking neural network for training and inference and we directly generate the adversarial noise on the frame-based data. Similar to previous experiments, we choose VGG-11 architecture and compare models trained from defense approaches including vanilla model, adversarial trained model and SR* strategy trained model. Here the regularization coefficient parameter is set to $\lambda=0.002$. 

Table \ref{tab:dvs} demonstrate the performance comparsion under adversarial attack with $\epsilon=0.031$. As can be seen from the table, the vanilla SNN can be easily fooled by the PGD50 attack and the robust performance is only 4.4\% under white-box PGD50 attack. However, with the application of adversarial training and gradient sparsity regularization, the robust performance shows a significant improvement. The adversarially trained model's robustness increases to 51.6\% and further rises to 61.2\% under PGD50 white-box attack. The SR* strategy model also exhibits a substantial improvement in adversarial robustness under black-box attacks, elevating the classification accuracy from 20.7\% to 68.9\% when subjected to the PGD50. These successful defenses of the sparse gradient method on the Dynamic Vision Sensor dataset proves the effectiveness and possibility of application of our method on neuromorphic datasets.

\begin{table}
\caption{Experiments on CIFAR10-DVS.}
\resizebox{1\linewidth}{!}
{
\begin{tabular}{ccccccccccc}
\toprule
\multirow{2}{*}{Defense}
& \multirow{2}{*}{Clean}
& \multicolumn{2}{c}{White Box Attack} 
& \multicolumn{2}{c}{Black Box Attack} \\
\cmidrule(lr){3-4} \cmidrule(lr){5-6}
& & FGSM & PGD50 & FGSM & PGD50 \\
\midrule
 Vanilla & 78.80 & 22.20 &  4.40 & 34.00 & 20.70 \\
 AT      & 76.80 & 60.00 & 51.60 & 71.50 & 65.50 \\
 SR      & 77.40 & 37.30 & 27.90 & 44.60 & 37.70 \\
 SR*     & 75.60 & 64.60 & 61.20 & 72.60 & 68.90 \\
\bottomrule
\end{tabular}
}
\label{tab:dvs}
\end{table}

\begin{table}[t]
\caption{Ablation study of the sparsity regularization.}
\label{tab:ablation}
\setlength\tabcolsep{2pt}
\centering
\resizebox{\linewidth}{!}{
\begin{threeparttable}
\begin{tabular}{llcccccc}
\toprule
SR & AT   & Clean & FGSM & RFGSM & PGD30 & PGD50 & APGD10\\ \midrule[1pt]
\multicolumn{8}{c}{CIFAR-10~~~~WRN-16} \\ \midrule
\usym{2715} & \usym{2715} & 93.89 & 5.23 & 3.43 & 0.00 & 0.00 & 0.00 \\ 
\usym{2715} & \Checkmark  & 90.97 & 33.49 & 58.19 & 14.89 & 14.62 & 9.13 \\ 
\Checkmark  & \usym{2715} & 86.57 & 34.79 & 55.96 & 12.27 & 11.70 & 8.25\\ 
\Checkmark  & \Checkmark  & 85.63 & 48.47 & 64.65 & 37.04 & 36.74   & 29.03 \\  \midrule[1pt]
\multicolumn{8}{c}{CIFAR-100~~~~WRN-16} \\ \midrule
\usym{2715} & \usym{2715} & 74.59 & 3.51 & 1.37 & 0.00 & 0.00 & 0.00 \\ 
\usym{2715} & \Checkmark  & 67.37 & 19.07 & 33.19 & 8.12 & 7.68  & 4.88 \\ 
\Checkmark  & \usym{2715} & 67.67 & 11.15 & 18.18 & 0.87 & 0.84 & 0.47  \\ 
\Checkmark  & \Checkmark  & 60.37 & 25.76 & 36.93 & 18.39 & 18.11 & 13.32 \\ 
\bottomrule
\end{tabular}
\end{threeparttable}
}
\end{table}

\subsection{Ablation Study of the Sparsity Regularization}
In the ablation study, we compare the robustness performance of SNNs using different training strategies: vanilla SNN, SR-SNN, AT-SNN, and SR*-SNN. The results on CIFAR-10 and CIFAR-100 are presented in Table~\ref{tab:ablation}, and the key findings are summarized as follows.

Firstly, it is crucial to note that vanilla SNNs exhibit poor adversarial robustness, with their classification accuracy dropping to a mere 5\% when subjected to the FGSM attack.
However, the SR strategy significantly enhances this performance, achieving a classification accuracy of 34.79\% on the CIFAR-10 dataset and 11.15\% on the CIFAR-100 dataset.
Furthermore, combining the SR strategy with the AT strategy further boosts the robustness of SNNs, particularly against strong attacks like PGD50 and APGD10,  resulting in a notable 10\%-30\% improvement.

Additionally, it is observed that the classification accuracy of robust SNNs on clean images may typically be slightly lower than that of baseline models. This phenomenon is consistent across all robustness training strategies. For example, WRN-16 models trained using any strategy exhibit a classification accuracy of less than 70\% on clean images in CIFAR-100.
Striking a balance between adversarial robustness and classification accuracy on clean images remains an open challenge in the field, warranting further exploration.

\subsection{Search for the Optimal Coefficient Parameter}

We conduct an extensive exploration to determine the optimal coefficient parameter $\lambda$, trying to strike a balance between robustness on adversarial images and classification accuracy on clean images.
The investigation specifically targets the CIFAR-10 dataset and the SNN model employs the WRN-16 architecture.

\begin{figure}[!t]  
    \centering
    \subfloat[Accuracy\label{fig:lambda_choose1}]{
    \includegraphics[width=0.48\linewidth]{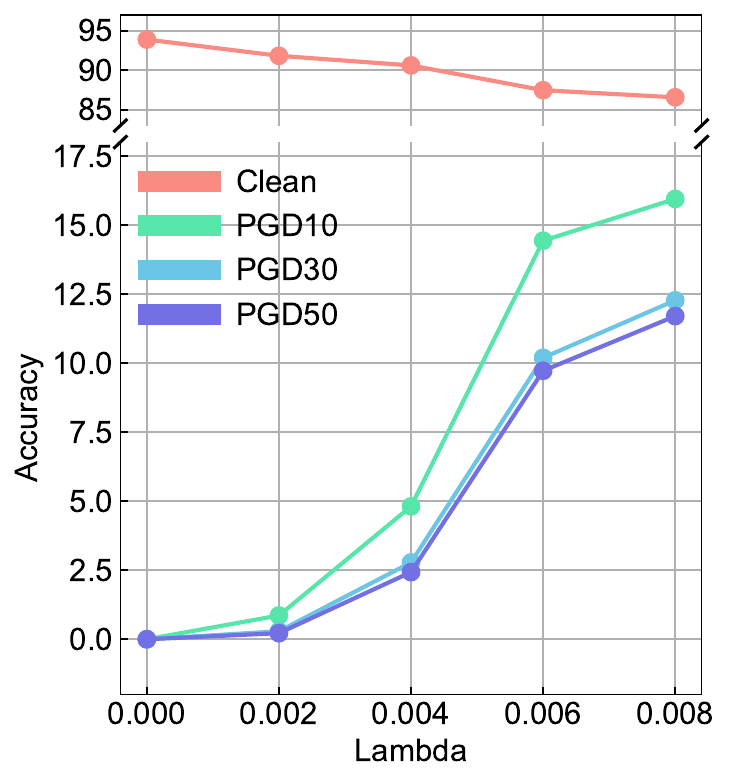}} 
    \subfloat[Gradient Norm\label{fig:lambda_choose2}]{
    \includegraphics[width=0.48\linewidth]{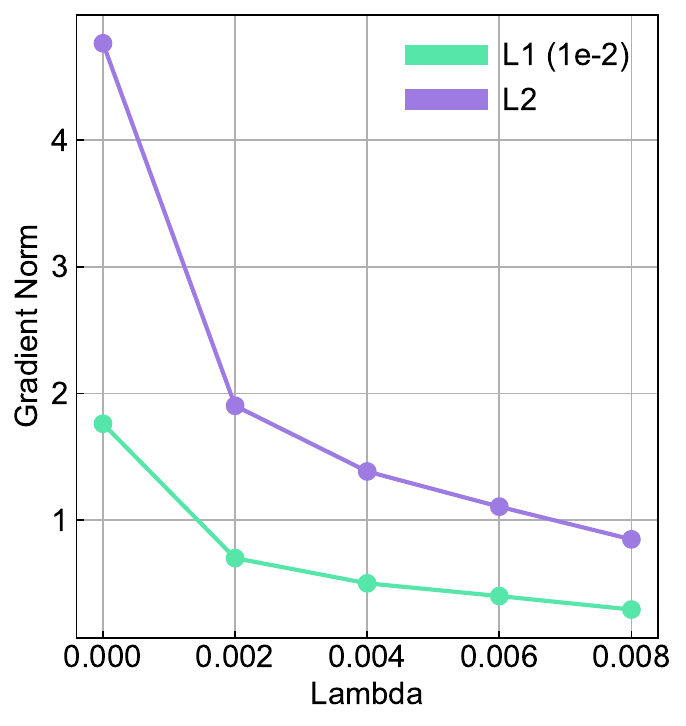}}
    \caption{The influence of the coefficient parameter $\lambda$ on classification accuracy and gradient sparsity. (a): Fluctuations in clean accuracy and adversarial accuracy under PGD attacks across different values of $\lambda$. (b): The $\ell_1$ and $\ell_2$ norms of the gradient with varying $\lambda$.}
    \label{fig:lambda_choose}
\end{figure} 

As described in Figure~\ref{fig:lambda_choose}~\subref{fig:lambda_choose1}, we test the impact of $\lambda$ varying within the range of 0.000 to 0.008. Notably,  increasing the value of $\lambda$ led to a decrease in classification accuracy on clean images but a significant improvement in adversarial robustness. To be specific, when using a coefficient parameter of $\lambda=0.008$, the classification accuracy under PGD10 attack increases from zero to 16\%, while maintaining over 85\% accuracy on clean images.

Figure~\ref{fig:lambda_choose}~\subref{fig:lambda_choose2} provides a visual representation of the effect of $\lambda$ on gradient sparsity after training. We computed the average $\ell_1$ and $\ell_2$ norm of $\nabla_{\bm{x}} f_y$ over the test dataset using models trained with different $\lambda$. Values of both the $\ell_1$ norm and $\ell_2$ decrease significantly as the coefficient parameter increases, indicating the correctness of the approximation method introduced in Proposition~\ref{prop:L1approx} and the effectiveness of the gradient sparsity regularization.

Based on these findings, we select $\lambda=0.008$ to train the SR-WRN-16 on the CIFAR-10 dataset to strike a balance between clean accuracy and adversarial robustness.
It is worth noting that the optimal choice of $\lambda$ may vary for different datasets. For additional insights into the relationship between $\lambda$, clean accuracy, and adversarial robustness on the CIFAR-100 dataset, please refer to the line chart presented in Appendix~\ref{append:lambda_cifar100}.

\subsection{Visualization of Gradient Sparsity}

\begin{figure}[!t]  
    \centering
    \includegraphics[width=0.9\linewidth]{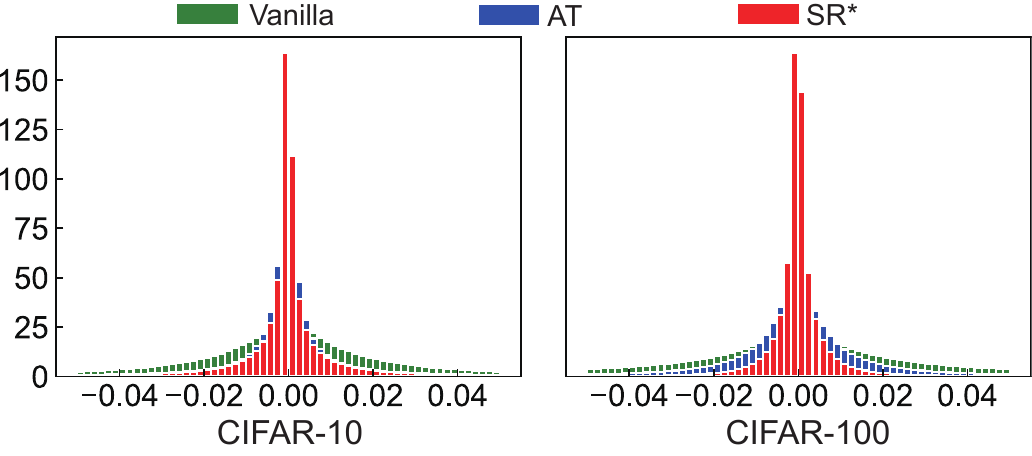}
    \caption{The normalized distribution of $\nabla_{\bm{x}} f_y(\bm{x})$. }
    \label{fig:gradient_sparsity}
\end{figure}

To validate the effectiveness of the proposed approximation method (Proposition~\ref{prop:L1approx}), we compute the gradient
$\nabla_{\bm{x}} f_y(\bm{x})$ at $\bm{x}$ in three cases: $f$ is a vanilla SNN, $f$ is an SNN with adversarial training (AT), and $f$ is an SNN trained with the proposed gradient sparsity regularization and adversarial training (SR*).
Figure~\ref{fig:gradient_sparsity} illustrates the overall distribution of components in the gradient across all test samples in the CIFAR-10 and CIFAR-100 datasets, respectively.

The results clearly show that the distribution of gradient components' values  for SR*-SNNs is more concentrated around zero compared to that of vanilla SNNs and AT-SNNs.
This indicates that SR*-SNNs exhibit sparser gradients with respect to the input image, demonstrating the effectiveness of the finite difference method proposed in Proposition~\ref{prop:L1approx} in constraining gradient sparsity. Meanwhile, these findings suggest a correlation between the sparsity of gradients and the robustness of SNNs to some extent: sparser gradients contribute to the enhancement of SNN robustness. 

\begin{figure*}[t]
    \centering
    \includegraphics[width=0.9\linewidth]{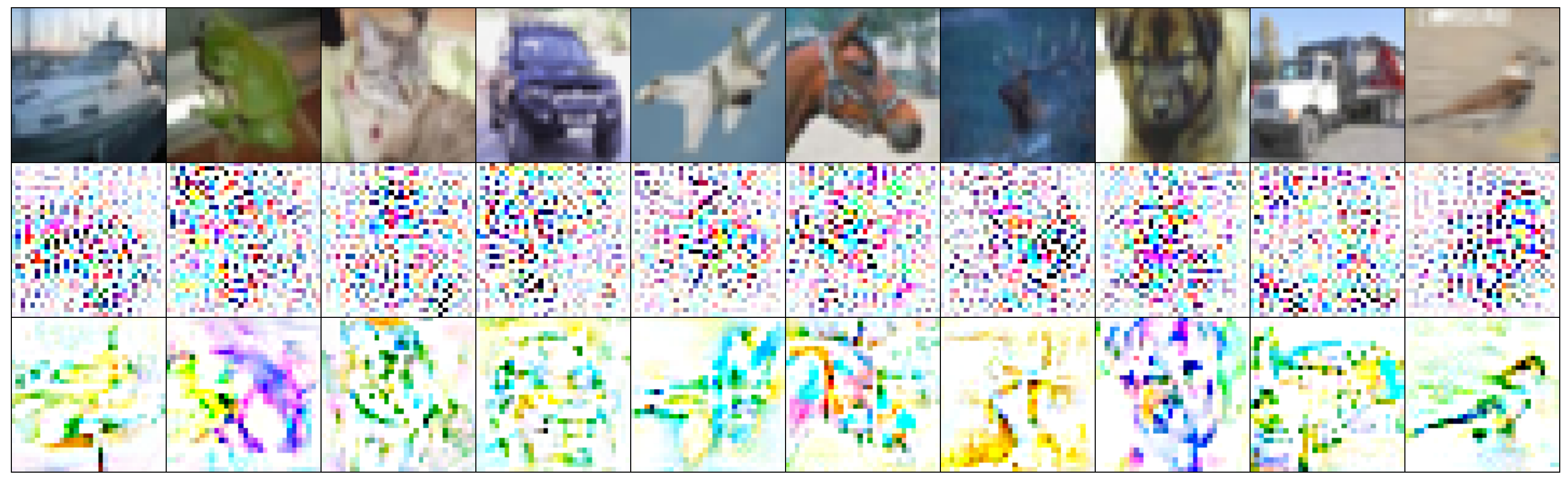}
    \caption{Heatmaps of $\nabla_{\bm{x}} f_y$ where $f$ is a villain SNN (top) or an SR-SNN (down).}
    \label{fig:gradient_visual}
\end{figure*}

To further substantiate the claim that SR-SNNs possess sparser gradients compared to vanilla SNNs, we present heatmaps of $\nabla_{\bm{x}} f_y$ for several examples from CIFAR-10~\cite{tsipras2019robustness}. In Figure~\ref{fig:gradient_visual}, the first row displays some original images selected from CIFAR-10, while the second and third rows show the corresponding heatmaps of $\nabla_{\bm{x}} f_y$ for vanilla SNN and SR-SNN, respectively.

Notably, the points on the heatmap of the vanilla SNN are densely distributed, while the points on the heatmap of the SR-SNN are more sparsely arranged. Moreover, the heatmap of gradients of the vanilla SNN appears cluttered to reflect any information about the image. However, the heatmap of the gradient of the SR-SNN shows some clear texture information of the image, which is beneficial to the interpretability of SNN. Therefore, we infer that the gradient sparsity regularization can not only improve the robustness of SNNs, but also provide some interpretability for SNNs.

\subsection{Impact of SR strategy to the Robustness Under Random Attacks}

We also investigate the impact of the SR strategy on the robustness of SNNs under random attacks. Experiments are conducted on CIFAR datasets using the WRN-16 architecture as the baseline. Random perturbations are uniformly drawn from a hyper-cube $\{ \delta_\text{rand}:\Vert \delta_\text{rand} \Vert_\infty \leqslant \epsilon\}$, with $\epsilon=0.1$. The classification accuracy of the vanilla WRN-16, SR-trained WRN-16, and SR*-trained (PGD5+SR) WRN-16 under random attacks is presented in Table~\ref{tab:random_robustness}. 

According to the table, both the SR and SR* strategies significantly enhance the robustness of SNNs against random attacks. For example, the single SR strategy improves the classification accuracy by nearly 20\% on the CIFAR-10 dataset. When combined with adversarial training, the SR* strategy still increases the random robustness, achieving 81\% classification accuracy. This indicates that the SR strategy does not compromise robustness under random attacks while narrowing the gap between adversarial and random vulnerabilities.

In addition to the experiments reported in the main manuscript, we also analyze the computational cost of the SR strategy in Appendix~\ref{append:cost}. We find that the training time for the SR strategy is less than that of PGD5 adversarial training, indicating that the SR strategy is efficient. More detailed explanations can be found in the appendix.

\section{Conclusion and Limitation}

\noindent \textbf{Conclusion.}
This paper introduces a new perspective on SNN robustness by incorporating the concept of gradient sparsity. We theoretically prove that the ratio of adversarial vulnerability to random vulnerability in SNNs is upper bounded by the sparsity of the true label probability with respect to the input image. 
Moreover, we propose the SR training strategy to train robust SNNs against adversarial attacks. Experimental results confirm the consistency between the theoretical analysis and practical tests.
This insight is expected to inspire ongoing research focused on enhancing SNN robustness by reducing gradient sparsity. Furthermore, it may also spark interest in investigating the robustness of event-driven SNNs, which naturally exhibit strong spike sparsity.

\begin{table}[t]
\centering
\caption{Classification accuracy (\%) of models trained with different methods under random attacks on CIFAR datasets.}
\label{tab:random_robustness}
\resizebox{0.7\linewidth}{!}{
\begin{tabular}{cccc}
\toprule
   & Vanilla & SR     & SR*    \\
\midrule
CIFAR-10 & 67.467  & 86.327 & 81.885 \\
CIFAR-100 & 26.270  & 49.900 & 48.678  \\
\bottomrule
\end{tabular}
}
\end{table}

\noindent \textbf{Limitation.}
The limitation of this work is that the improvement in adversarial robustness achieved by the SR strategy comes at the cost of a notable accuracy loss on clean images. In future work, we aim to strike a better balance between classification accuracy and adversarial robustness. For instance, we may employ the simulated annealing algorithm or structure learning~\cite{Wolfgang_2018_ICLR,Xu_2023_AAAI} in the SR strategy. Additionally, considering that biological vision has strong robustness~\cite{dapello2020simulating,yu2024robust}, proposing new SNN models by drawing inspiration from the mechanisms of biological vision is an important direction for the future.

\section*{Acknowledgements}
This work was supported by the National Natural Science Foundation of China (62176003, 62088102) and the Beijing Nova Program (20230484362). We also acknowledge Lei Wu for valuable discussions.

\section*{Impact Statement}

Our research makes contributions to the field of SNN by addressing the importance of gradient sparsity. This work provides both advanced theoretical understanding and practical solutions for robust SNN training, with a deliberate focus on ensuring no discernible negative societal consequences.

\bibliography{example_paper}

\begin{thebibliography}{79}
\providecommand{\natexlab}[1]{#1}
\providecommand{\url}[1]{\texttt{#1}}
\expandafter\ifx\csname urlstyle\endcsname\relax
  \providecommand{\doi}[1]{doi: #1}\else
  \providecommand{\doi}{doi: \begingroup \urlstyle{rm}\Url}\fi

\bibitem[Bellec et~al.(2018)Bellec, Kappel, Maass, and Legenstein]{Wolfgang_2018_ICLR}
Bellec, G., Kappel, D., Maass, W., and Legenstein, R.
\newblock Deep rewiring: {Training} very sparse deep networks.
\newblock In \emph{International Conference on Learning Representations}, pp.\  1--24, 2018.

\bibitem[Bohte et~al.(2000)Bohte, Kok, and La~Poutr{\'e}]{bohte2000spikeprop}
Bohte, S.~M., Kok, J.~N., and La~Poutr{\'e}, J.~A.
\newblock Spikeprop: Backpropagation for networks of spiking neurons.
\newblock In \emph{European Symposium on Artificial Neural Networks}, pp.\  419--424, 2000.

\bibitem[Bu et~al.(2022)Bu, Fang, Ding, Dai, Yu, and Huang]{bu2021optimal}
Bu, T., Fang, W., Ding, J., Dai, P., Yu, Z., and Huang, T.
\newblock Optimal {ANN-SNN} conversion for high-accuracy and ultra-low-latency spiking neural networks.
\newblock In \emph{International Conference on Learning Representations}, pp.\  1--19, 2022.

\bibitem[Bu et~al.(2023)Bu, Ding, Hao, and Yu]{bu2023rate}
Bu, T., Ding, J., Hao, Z., and Yu, Z.
\newblock Rate gradient approximation attack threats deep spiking neural networks.
\newblock In \emph{IEEE Conference on Computer Vision and Pattern Recognition}, pp.\  7896--7906, 2023.

\bibitem[Cao et~al.(2015)Cao, Chen, and Khosla]{cao2015spiking}
Cao, Y., Chen, Y., and Khosla, D.
\newblock Spiking deep convolutional neural networks for energy-efficient object recognition.
\newblock \emph{International Journal of Computer Vision}, 113\penalty0 (1):\penalty0 54--66, 2015.

\bibitem[Croce \& Hein(2020)Croce and Hein]{croce2020reliable}
Croce, F. and Hein, M.
\newblock Reliable evaluation of adversarial robustness with an ensemble of diverse parameter-free attacks.
\newblock In \emph{International Conference on Machine Learning}, pp.\  2206--2216, 2020.

\bibitem[Dapello et~al.(2020)Dapello, Marques, Schrimpf, Geiger, Cox, and DiCarlo]{dapello2020simulating}
Dapello, J., Marques, T., Schrimpf, M., Geiger, F., Cox, D., and DiCarlo, J.~J.
\newblock Simulating a primary visual cortex at the front of cnns improves robustness to image perturbations.
\newblock In \emph{Advances in Neural Information Processing Systems}, pp.\  13073--13087, 2020.

\bibitem[Deng \& Gu(2021)Deng and Gu]{deng2020optimal}
Deng, S. and Gu, S.
\newblock Optimal conversion of conventional artificial neural networks to spiking neural networks.
\newblock In \emph{International Conference on Learning Representations}, pp.\  1--14, 2021.

\bibitem[Deng et~al.(2022)Deng, Li, Zhang, and Gu]{deng2021temporal}
Deng, S., Li, Y., Zhang, S., and Gu, S.
\newblock Temporal efficient training of spiking neural network via gradient re-weighting.
\newblock In \emph{International Conference on Learning Representations}, pp.\  1--15, 2022.

\bibitem[Diehl et~al.(2015)Diehl, Neil, Binas, Cook, Liu, and Pfeiffer]{diehl2015fast}
Diehl, P.~U., Neil, D., Binas, J., Cook, M., Liu, S.-C., and Pfeiffer, M.
\newblock Fast-classifying, high-accuracy spiking deep networks through weight and threshold balancing.
\newblock In \emph{International Joint Conference on Neural Networks}, pp.\  1--8, 2015.

\bibitem[Ding et~al.(2022)Ding, Bu, Yu, Huang, and Liu]{ding2022snn}
Ding, J., Bu, T., Yu, Z., Huang, T., and Liu, J.~K.
\newblock {SNN-RAT}: Robustness-enhanced spiking neural network through regularized adversarial training.
\newblock In \emph{Advances in Neural Information Processing Systems}, pp.\  1--14, 2022.

\bibitem[Duan et~al.(2022)Duan, Ding, Chen, Yu, and Huang]{duantemporal}
Duan, C., Ding, J., Chen, S., Yu, Z., and Huang, T.
\newblock Temporal effective batch normalization in spiking neural networks.
\newblock In \emph{Advances in Neural Information Processing Systems}, pp.\  34377--34390, 2022.

\bibitem[Esser et~al.(2016)Esser, Merolla, Arthur, Cassidy, Appuswamy, Andreopoulos, Berg, McKinstry, Melano, Barch, di~Nolfo, Datta, Amir, Taba, Flickner, and Modha]{esser2016convolutional}
Esser, S.~K., Merolla, P.~A., Arthur, J.~V., Cassidy, A.~S., Appuswamy, R., Andreopoulos, A., Berg, D.~J., McKinstry, J.~L., Melano, T., Barch, D.~R., di~Nolfo, C., Datta, P., Amir, A., Taba, B., Flickner, M.~D., and Modha, D.~S.
\newblock Convolutional networks for fast, energy-efficient neuromorphic computing.
\newblock \emph{The Proceedings of the National Academy of Sciences}, 113\penalty0 (41):\penalty0 11441--11446, 2016.

\bibitem[Fang et~al.(2021{\natexlab{a}})Fang, Yu, Chen, Huang, Masquelier, and Tian]{fang2021deep}
Fang, W., Yu, Z., Chen, Y., Huang, T., Masquelier, T., and Tian, Y.
\newblock Deep residual learning in spiking neural networks.
\newblock In \emph{Advances in Neural Information Processing Systems}, pp.\  21056--21069, 2021{\natexlab{a}}.

\bibitem[Fang et~al.(2021{\natexlab{b}})Fang, Yu, Chen, Masquelier, Huang, and Tian]{fang2021incorporating}
Fang, W., Yu, Z., Chen, Y., Masquelier, T., Huang, T., and Tian, Y.
\newblock Incorporating learnable membrane time constant to enhance learning of spiking neural networks.
\newblock In \emph{International Conference on Computer Vision}, pp.\  2641--2651, 2021{\natexlab{b}}.

\bibitem[Finlay \& Oberman(2021)Finlay and Oberman]{2021_Chris_GradientApproximation}
Finlay, C. and Oberman, A.~M.
\newblock Scaleable input gradient regularization for adversarial robustness.
\newblock \emph{Machine Learning with Applications}, 3:\penalty0 100017, 2021.

\bibitem[Goodfellow et~al.(2015)Goodfellow, Shlens, and Szegedy]{goodfellow2014explaining}
Goodfellow, I.~J., Shlens, J., and Szegedy, C.
\newblock Explaining and harnessing adversarial examples.
\newblock In \emph{International Conference on Learning Representations}, pp.\  1--11, 2015.

\bibitem[Guo et~al.(2022)Guo, Chen, Zhang, Liu, Wang, Huang, and Ma]{guo2022loss}
Guo, Y., Chen, Y., Zhang, L., Liu, X., Wang, Y., Huang, X., and Ma, Z.
\newblock {IM}-loss: Information maximization loss for spiking neural networks.
\newblock In \emph{Advances in Neural Information Processing Systems}, pp.\  156--166, 2022.

\bibitem[Han et~al.(2020)Han, Srinivasan, and Roy]{Han_2020_CVPR}
Han, B., Srinivasan, G., and Roy, K.
\newblock {RMP-SNN}: Residual membrane potential neuron for enabling deeper high-accuracy and low-latency spiking neural network.
\newblock In \emph{{IEEE} Conference on Computer Vision and Pattern Recognition}, pp.\  13555--13564, 2020.

\bibitem[Hao et~al.(2023{\natexlab{a}})Hao, Bu, Shi, Huang, Yu, and Huang]{haothreaten}
Hao, Z., Bu, T., Shi, X., Huang, Z., Yu, Z., and Huang, T.
\newblock Threaten spiking neural networks through combining rate and temporal information.
\newblock In \emph{International Conference on Learning Representations}, pp.\  1--17, 2023{\natexlab{a}}.

\bibitem[Hao et~al.(2023{\natexlab{b}})Hao, Ding, Bu, Huang, and Yu]{hao2023bridging}
Hao, Z., Ding, J., Bu, T., Huang, T., and Yu, Z.
\newblock Bridging the gap between anns and snns by calibrating offset spikes.
\newblock In \emph{International Conference on Learning Representations}, 2023{\natexlab{b}}.

\bibitem[He et~al.(2016)He, Zhang, Ren, and Sun]{2016_CVPR_He_ResNet}
He, K., Zhang, X., Ren, S., and Sun, J.
\newblock Deep residual learning for image recognition.
\newblock In \emph{{IEEE} Conference on Computer Vision and Pattern Recognition}, pp.\  770--778, 2016.

\bibitem[Ho et~al.(2022)Ho, Lee, and Kang]{2022_AI_Ho_AdvTrain}
Ho, J., Lee, B., and Kang, D.
\newblock Attack-less adversarial training for a robust adversarial defense.
\newblock \emph{Applied Intelligence}, 52\penalty0 (4):\penalty0 4364--4381, 2022.

\bibitem[Ho \& Chang(2021)Ho and Chang]{ho2021tcl}
Ho, N.-D. and Chang, I.-J.
\newblock {TCL}: an {ANN-to-SNN} conversion with trainable clipping layers.
\newblock In \emph{Design Automation Conference}, pp.\  793--798, 2021.

\bibitem[Hu et~al.(2023)Hu, Tang, and Pan]{hu2023spiking}
Hu, Y., Tang, H., and Pan, G.
\newblock Spiking deep residual networks.
\newblock \emph{IEEE Transactions on Neural Networks and Learning Systems}, 34\penalty0 (8):\penalty0 5200--5205, 2023.

\bibitem[Khan et~al.(2023)Khan, El{-}Sayed, Malik, Zia, Khan, Alkaabi, and Ignatious]{2021_CompSur_Khan_AutoDrive}
Khan, M.~A., El{-}Sayed, H., Malik, S., Zia, M.~T., Khan, J., Alkaabi, N., and Ignatious, H.~A.
\newblock Level-5 autonomous driving - {Are} we there yet? {A} review of research literature.
\newblock \emph{ACM Computing Surveys}, 55\penalty0 (2):\penalty0 1--38, 2023.

\bibitem[Kim \& Panda(2021)Kim and Panda]{kim2020revisiting}
Kim, Y. and Panda, P.
\newblock Revisiting batch normalization for training low-latency deep spiking neural networks from scratch.
\newblock \emph{Frontiers in Neuroscience}, 15:\penalty0 773954, 2021.

\bibitem[Kim et~al.(2022)Kim, Park, Moitra, Bhattacharjee, Venkatesha, and Panda]{kim2022rate}
Kim, Y., Park, H., Moitra, A., Bhattacharjee, A., Venkatesha, Y., and Panda, P.
\newblock Rate coding or direct coding: Which one is better for accurate, robust, and energy-efficient spiking neural networks?
\newblock In \emph{IEEE International Conference on Acoustics, Speech and Signal Processing}, pp.\  71--75, 2022.

\bibitem[Kundu et~al.(2021)Kundu, Pedram, and Beerel]{kundu2021hire}
Kundu, S., Pedram, M., and Beerel, P.~A.
\newblock {Hire-SNN}: Harnessing the inherent robustness of energy-efficient deep spiking neural networks by training with crafted input noise.
\newblock In \emph{International Conference on Computer Vision}, pp.\  5189--5198, 2021.

\bibitem[Lee et~al.(2020)Lee, Sarwar, Panda, Srinivasan, and Roy]{lee2020enabling}
Lee, C., Sarwar, S.~S., Panda, P., Srinivasan, G., and Roy, K.
\newblock Enabling spike-based backpropagation for training deep neural network architectures.
\newblock \emph{Frontiers in Neuroscience}, 14\penalty0 (119):\penalty0 1--22, 2020.

\bibitem[Leontev et~al.(2021)Leontev, Antonov, and Sukhov]{leontev2021robustness}
Leontev, M., Antonov, D., and Sukhov, S.
\newblock Robustness of spiking neural networks against adversarial attacks.
\newblock In \emph{International Conference on Information Technology and Nanotechnology}, pp.\  1--6, 2021.

\bibitem[Li et~al.(2021)Li, Deng, Dong, Gong, and Gu]{li2021free}
Li, Y., Deng, S., Dong, X., Gong, R., and Gu, S.
\newblock A free lunch from {ANN}: Towards efficient, accurate spiking neural networks calibration.
\newblock In \emph{International Conference on Machine Learning}, pp.\  6316--6325, 2021.

\bibitem[Liang et~al.(2021)Liang, Hu, Deng, Wu, Li, Ding, Li, and Xie]{liang2021exploring}
Liang, L., Hu, X., Deng, L., Wu, Y., Li, G., Ding, Y., Li, P., and Xie, Y.
\newblock Exploring adversarial attack in spiking neural networks with spike-compatible gradient.
\newblock \emph{IEEE Transactions on Neural Networks and Learning Systems}, 34\penalty0 (5):\penalty0 2569--2583, 2021.

\bibitem[Liang et~al.(2022)Liang, Xu, Hu, Deng, and Xie]{liang2022toward}
Liang, L., Xu, K., Hu, X., Deng, L., and Xie, Y.
\newblock Toward robust spiking neural network against adversarial perturbation.
\newblock In \emph{Advances in Neural Information Processing Systems}, pp.\  10244--10256, 2022.

\bibitem[Lipton(2018)]{2018_Zachary_Inter}
Lipton, Z.~C.
\newblock The mythos of model interpretability.
\newblock \emph{Communications of the {ACM}}, 61\penalty0 (10):\penalty0 36--43, 2018.

\bibitem[Liu et~al.(2022)Liu, Jiang, and Jiang]{2022_CS_Liu_LPIPISAttack}
Liu, Y., Jiang, M., and Jiang, T.
\newblock Transferable adversarial examples based on global smooth perturbations.
\newblock \emph{Computers \& Security}, 121:\penalty0 102816, 2022.

\bibitem[Liu et~al.(2021)Liu, Lin, Cao, Hu, Wei, Zhang, Lin, and Guo]{2021_Liu_SwinTrans}
Liu, Z., Lin, Y., Cao, Y., Hu, H., Wei, Y., Zhang, Z., Lin, S., and Guo, B.
\newblock Swin transformer: Hierarchical vision transformer using shifted windows.
\newblock In \emph{International Conference on Computer Vision}, pp.\  9992--10002, 2021.

\bibitem[Loshchilov \& Hutter(2017)Loshchilov and Hutter]{loshchilov2016sgdr}
Loshchilov, I. and Hutter, F.
\newblock {SGDR}: Stochastic gradient descent with warm restarts.
\newblock In \emph{International Conference on Learning Representations}, pp.\  1--16, 2017.

\bibitem[Maass(1997)]{maas1997networks}
Maass, W.
\newblock Networks of spiking neurons: The third generation of neural network models.
\newblock \emph{Neural Networks}, 10\penalty0 (9):\penalty0 1659--1671, 1997.

\bibitem[Madry et~al.(2018)Madry, Makelov, Schmidt, Tsipras, and Vladu]{aleksander2018towards}
Madry, A., Makelov, A., Schmidt, L., Tsipras, D., and Vladu, A.
\newblock Towards deep learning models resistant to adversarial attacks.
\newblock In \emph{International Conference on Learning Representations}, pp.\  1--23, 2018.

\bibitem[Marchisio et~al.(2021{\natexlab{a}})Marchisio, Pira, Martina, Masera, and Shafique]{marchisio2021dvs}
Marchisio, A., Pira, G., Martina, M., Masera, G., and Shafique, M.
\newblock {DVS-Attacks}: {Adversarial} attacks on dynamic vision sensors for spiking neural networks.
\newblock In \emph{International Joint Conference on Neural Networks}, pp.\  1--9, 2021{\natexlab{a}}.

\bibitem[Marchisio et~al.(2021{\natexlab{b}})Marchisio, Pira, Martina, Masera, and Shafique]{marchisio2021r}
Marchisio, A., Pira, G., Martina, M., Masera, G., and Shafique, M.
\newblock {R-SNN}: An analysis and design methodology for robustifying spiking neural networks against adversarial attacks through noise filters for dynamic vision sensors.
\newblock In \emph{{IEEE/RSJ} International Conference on Intelligent Robots and Systems}, pp.\  6315--6321, 2021{\natexlab{b}}.

\bibitem[Mishra et~al.(2021)Mishra, Dash, and Jena]{2021_NeuroComput_Mishra_Diagnosis}
Mishra, S., Dash, A., and Jena, L.
\newblock Use of deep learning for disease detection and diagnosis.
\newblock In \emph{Bio-inspired Neurocomputing}, pp.\  181--201. Springer, Singapore, 2021.

\bibitem[Mostafa(2017)]{mostafa2017supervised}
Mostafa, H.
\newblock Supervised learning based on temporal coding in spiking neural networks.
\newblock \emph{IEEE Transactions on Neural Networks and Learning Systems}, 29\penalty0 (7):\penalty0 3227--3235, 2017.

\bibitem[Natarajan(1995)]{1995_SIAM_natarajan_L0Linear}
Natarajan, B.~K.
\newblock Sparse approximate solutions to linear systems.
\newblock \emph{SIAM Journal on Computing}, 24\penalty0 (2):\penalty0 227--234, 1995.

\bibitem[Neftci et~al.(2019)Neftci, Mostafa, and Zenke]{neftci2019surrogate}
Neftci, E.~O., Mostafa, H., and Zenke, F.
\newblock Surrogate gradient learning in spiking neural networks: Bringing the power of gradient-based optimization to spiking neural networks.
\newblock \emph{IEEE Signal Processing Magazine}, 36\penalty0 (6):\penalty0 51--63, 2019.

\bibitem[Nomura et~al.(2022)Nomura, Sakemi, Hosomi, and Morie]{nomura2022robustness}
Nomura, O., Sakemi, Y., Hosomi, T., and Morie, T.
\newblock Robustness of spiking neural networks based on time-to-first-spike encoding against adversarial attacks.
\newblock \emph{IEEE Transactions on Circuits and Systems II: Express Briefs}, 69\penalty0 (9):\penalty0 3640--3644, 2022.

\bibitem[{\"{O}}zdenizci \& Legenstein(2023){\"{O}}zdenizci and Legenstein]{Ozan2023adversarial}
{\"{O}}zdenizci, O. and Legenstein, R.
\newblock Adversarially robust spiking neural networks through conversion.
\newblock \emph{arXiv preprint arXiv:2311.09266}, 2023.

\bibitem[Ramirez et~al.(2013)Ramirez, Kreinovich, and Argaez]{2013_Ramirez_L1ApproxL0}
Ramirez, C., Kreinovich, V., and Argaez, M.
\newblock Why $\ell_1$ is a good approximation to $\ell_0$: A geometric explanation.
\newblock \emph{Journal of Uncertain Systems}, 7:\penalty0 203--207, 2013.

\bibitem[Rathi \& Roy(2021)Rathi and Roy]{rathi2021diet}
Rathi, N. and Roy, K.
\newblock {DIET-SNN}: A low-latency spiking neural network with direct input encoding and leakage and threshold optimization.
\newblock \emph{IEEE Transactions on Neural Networks and Learning Systems}, 34\penalty0 (6):\penalty0 3174--3182, 2021.

\bibitem[Roy et~al.(2019)Roy, Jaiswal, and Panda]{roy2019towards}
Roy, K., Jaiswal, A., and Panda, P.
\newblock Towards spike-based machine intelligence with neuromorphic computing.
\newblock \emph{Nature}, 575\penalty0 (7784):\penalty0 607--617, 2019.

\bibitem[Sengupta et~al.(2019)Sengupta, Ye, Wang, Liu, and Roy]{sengupta2019going}
Sengupta, A., Ye, Y., Wang, R., Liu, C., and Roy, K.
\newblock Going deeper in spiking neural networks: {VGG} and residual architectures.
\newblock \emph{Frontiers in Neuroscience}, 13\penalty0 (95):\penalty0 1--10, 2019.

\bibitem[Shao et~al.(2023)Shao, Fang, Li, Feng, Shen, and Xu]{Xu_2023_NIPS}
Shao, Z., Fang, X., Li, Y., Feng, C., Shen, J., and Xu, Q.
\newblock {EICIL:} joint excitatory inhibitory cycle iteration learning for deep spiking neural networks.
\newblock In \emph{Advances in Neural Information Processing Systems}, pp.\  1--12, 2023.

\bibitem[Sharmin et~al.(2019)Sharmin, Panda, Sarwar, Lee, Ponghiran, and Roy]{sharmin2019comprehensive}
Sharmin, S., Panda, P., Sarwar, S.~S., Lee, C., Ponghiran, W., and Roy, K.
\newblock A comprehensive analysis on adversarial robustness of spiking neural networks.
\newblock In \emph{International Joint Conference on Neural Networks}, pp.\  1--8, 2019.

\bibitem[Sharmin et~al.(2020)Sharmin, Rathi, Panda, and Roy]{sharmin2020inherent}
Sharmin, S., Rathi, N., Panda, P., and Roy, K.
\newblock Inherent adversarial robustness of deep spiking neural networks: Effects of discrete input encoding and non-linear activations.
\newblock In \emph{European Conference on Computer Vision}, pp.\  399--414, 2020.

\bibitem[Shen et~al.(2023)Shen, Xu, Liu, Wang, Pan, and Tang]{Xu_2023_AAAI}
Shen, J., Xu, Q., Liu, J.~K., Wang, Y., Pan, G., and Tang, H.
\newblock {ESL-SNNs}: An evolutionary structure learning strategy for spiking neural networks.
\newblock In \emph{Proceedings of the AAAI Conference on Artificial Intelligence}, pp.\  86--93, 2023.

\bibitem[Simonyan \& Zisserman(2014)Simonyan and Zisserman]{simonyan2014very}
Simonyan, K. and Zisserman, A.
\newblock Very deep convolutional networks for large-scale image recognition.
\newblock \emph{arXiv preprint arXiv:1409.1556}, 2014.

\bibitem[Stutz et~al.(2019)Stutz, Hein, and Schiele]{2019_CVPR_Stutz_OffManifold}
Stutz, D., Hein, M., and Schiele, B.
\newblock Disentangling adversarial robustness and generalization.
\newblock In \emph{{IEEE} Conference on Computer Vision and Pattern Recognition}, pp.\  6976--6987, 2019.

\bibitem[Tanay \& Griffin(2016)Tanay and Griffin]{2016_arxiv_Tanay_LinearCase}
Tanay, T. and Griffin, L.~D.
\newblock A boundary tilting persepective on the phenomenon of adversarial examples.
\newblock \emph{arXiv}, arXiv preprint arXiv:1608.07690, 2016.

\bibitem[Tram{\`e}r et~al.(2018)Tram{\`e}r, Kurakin, Papernot, Goodfellow, Boneh, and McDaniel]{tramer2018ensemble}
Tram{\`e}r, F., Kurakin, A., Papernot, N., Goodfellow, I., Boneh, D., and McDaniel, P.
\newblock Ensemble adversarial training: Attacks and defenses.
\newblock In \emph{International Conference on Learning Representations}, pp.\  1--20, 2018.

\bibitem[Tsipras et~al.(2019)Tsipras, Santurkar, Engstrom, Turner, and Madry]{tsipras2019robustness}
Tsipras, D., Santurkar, S., Engstrom, L., Turner, A., and Madry, A.
\newblock Robustness may be at odds with accuracy.
\newblock In \emph{International Conference on Learning Representations}, pp.\  1--23, 2019.

\bibitem[Wong \& Kolter(2018)Wong and Kolter]{2018_ICML_Wong_DefenseConvex}
Wong, E. and Kolter, J.~Z.
\newblock Provable defenses against adversarial examples via the convex outer adversarial polytope.
\newblock In \emph{International Conference on Machine Learning}, pp.\  5283--5292, 2018.

\bibitem[Wong et~al.(2020)Wong, Rice, and Kolter]{wong2019fast}
Wong, E., Rice, L., and Kolter, J.~Z.
\newblock Fast is better than free: Revisiting adversarial training.
\newblock In \emph{International Conference on Learning Representations}, pp.\  1--12, 2020.

\bibitem[Wu et~al.(2018)Wu, Deng, Li, Zhu, and Shi]{wu2018STBP}
Wu, Y., Deng, L., Li, G., Zhu, J., and Shi, L.
\newblock Spatio-temporal backpropagation for training high-performance spiking neural networks.
\newblock \emph{Frontiers in Neuroscience}, 12\penalty0 (331):\penalty0 1--12, 2018.

\bibitem[Wu et~al.(2019)Wu, Deng, Li, Zhu, Xie, and Shi]{wu2019direct}
Wu, Y., Deng, L., Li, G., Zhu, J., Xie, Y., and Shi, L.
\newblock Direct training for spiking neural networks: Faster, larger, better.
\newblock In \emph{Proceedings of the AAAI Conference on Artificial Intelligence}, pp.\  1311--1318, 2019.

\bibitem[Xu et~al.(2020)Xu, Shi, Zhang, Wang, Chang, Huang, Kailkhura, Lin, and Hsieh]{2020_NIPS_Xu_Certified}
Xu, K., Shi, Z., Zhang, H., Wang, Y., Chang, K., Huang, M., Kailkhura, B., Lin, X., and Hsieh, C.
\newblock Automatic perturbation analysis for scalable certified robustness and beyond.
\newblock In \emph{Advances in Neural Information Processing Systems}, pp.\  1129--1141, 2020.

\bibitem[Xu et~al.(2022{\natexlab{a}})Xu, Mahmood, Fang, Rathbun, Ding, and Wen]{xu2022securing}
Xu, N., Mahmood, K., Fang, H., Rathbun, E., Ding, C., and Wen, W.
\newblock Securing the spike: On the transferabilty and security of spiking neural networks to adversarial examples.
\newblock \emph{arXiv preprint arXiv:2209.03358}, 2022{\natexlab{a}}.

\bibitem[Xu et~al.(2022{\natexlab{b}})Xu, Li, Shen, Zhang, Liu, Tang, and Pan]{Xu_2023_TNNLS}
Xu, Q., Li, Y., Shen, J., Zhang, P., Liu, J.~K., Tang, H., and Pan, G.
\newblock Hierarchical spiking-based model for efficient image classification with enhanced feature extraction and encoding.
\newblock \emph{IEEE Transactions on Neural Networks and Learning Systems}, pp.\  1--9, 2022{\natexlab{b}}.

\bibitem[Yao et~al.(2022)Yao, Li, Mo, and Cheng]{yao2022glif}
Yao, X., Li, F., Mo, Z., and Cheng, J.
\newblock {GLIF}: A unified gated leaky integrate-and-fire neuron for spiking neural networks.
\newblock In \emph{Advances in Neural Information Processing Systems}, pp.\  32160--32171, 2022.

\bibitem[Yu et~al.(2024)Yu, Bu, Zhang, Jia, Huang, and Liu]{yu2024robust}
Yu, Z., Bu, T., Zhang, Y., Jia, S., Huang, T., and Liu, J.~K.
\newblock Robust decoding of rich dynamical visual scenes with retinal spikes.
\newblock \emph{IEEE Transactions on Neural Networks and Learning Systems}, 2024.

\bibitem[Zagoruyko \& Komodakis(2016)Zagoruyko and Komodakis]{zagoruyko2016wide}
Zagoruyko, S. and Komodakis, N.
\newblock Wide residual networks.
\newblock In \emph{Procedings of the British Machine Vision Conference}, pp.\  1--15, 2016.

\bibitem[Zenke et~al.(2021)Zenke, Boht{\'e}, Clopath, Com{\c{s}}a, G{\"o}ltz, Maass, Masquelier, Naud, Neftci, Petrovici, et~al.]{zenke2021visualizing}
Zenke, F., Boht{\'e}, S.~M., Clopath, C., Com{\c{s}}a, I.~M., G{\"o}ltz, J., Maass, W., Masquelier, T., Naud, R., Neftci, E.~O., Petrovici, M.~A., et~al.
\newblock Visualizing a joint future of neuroscience and neuromorphic engineering.
\newblock \emph{Neuron}, 109\penalty0 (4):\penalty0 571--575, 2021.

\bibitem[Zhang et~al.(2021)Zhang, Cai, Lu, He, and Wang]{2021_ICML_Zhang_LinfNet}
Zhang, B., Cai, T., Lu, Z., He, D., and Wang, L.
\newblock Towards certifying $\ell_\infty$ robustness using neural networks with $\ell_\infty$-dist neurons.
\newblock In \emph{International Conference on Machine Learning}, pp.\  12368--12379, 2021.

\bibitem[Zhang et~al.(2020)Zhang, Chen, Xiao, Gowal, Stanforth, Li, Boning, and Hsieh]{2020_ICLR_Zhang_IBP}
Zhang, H., Chen, H., Xiao, C., Gowal, S., Stanforth, R., Li, B., Boning, D.~S., and Hsieh, C.
\newblock Towards stable and efficient training of verifiably robust neural networks.
\newblock In \emph{International Conference on Learning Representations}, pp.\  1--25, 2020.

\bibitem[Zhang et~al.(2022)Zhang, Wang, Wu, Belatreche, Amornpaisannon, Zhang, Miriyala, Qu, Chua, Carlson, et~al.]{zhang2022rectified}
Zhang, M., Wang, J., Wu, J., Belatreche, A., Amornpaisannon, B., Zhang, Z., Miriyala, V. P.~K., Qu, H., Chua, Y., Carlson, T.~E., et~al.
\newblock Rectified linear postsynaptic potential function for backpropagation in deep spiking neural networks.
\newblock \emph{IEEE Transactions on Neural Networks and Learning Systems}, 33\penalty0 (5):\penalty0 1947--1958, 2022.

\bibitem[Zhang \& Li(2020)Zhang and Li]{zhang2020temporal}
Zhang, W. and Li, P.
\newblock Temporal spike sequence learning via backpropagation for deep spiking neural networks.
\newblock In \emph{Advances in Neural Information Processing Systems}, pp.\  12022--12033, 2020.

\bibitem[Zheng et~al.(2021)Zheng, Wu, Deng, Hu, and Li]{zheng2021going}
Zheng, H., Wu, Y., Deng, L., Hu, Y., and Li, G.
\newblock Going deeper with directly-trained larger spiking neural networks.
\newblock In \emph{Proceedings of the AAAI Conference on Artificial Intelligence}, pp.\  11062--11070, 2021.

\bibitem[Zhu et~al.(2022)Zhu, Yu, Fang, Xie, Huang, and Masquelier]{zhutraining}
Zhu, Y., Yu, Z., Fang, W., Xie, X., Huang, T., and Masquelier, T.
\newblock Training spiking neural networks with event-driven backpropagation.
\newblock In \emph{Advances in Neural Information Processing Systems}, pp.\  30528--30541, 2022.

\bibitem[Zhu et~al.(2023)Zhu, Fang, Xie, Huang, and Yu]{zhu2024exploring}
Zhu, Y., Fang, W., Xie, X., Huang, T., and Yu, Z.
\newblock Exploring loss functions for time-based training strategy in spiking neural networks.
\newblock \emph{Advances in Neural Information Processing Systems}, pp.\  65366--65379, 2023.

\end{thebibliography}
\bibliographystyle{icml2024}

\newpage
\appendix
\onecolumn

\section{Proof of Theorem~\ref{thr:binary_sparsity}}
\label{append:proof_theorem}
\renewcommand{\thetheorem}{1}
\begin{theorem}
	Suppose $f$ represents an SNN and $\epsilon$ is the strength of an attack. Given an input image $\bm{x}$ with corresponding label $y$, the ratio of $\rho_\text{adv}(f, \bm{x},\epsilon, \ell_\infty)$ and $ \rho_\text{rand}(f, \bm{x},\epsilon, \ell_\infty)$ is upper bounded by the sparsity of $\nabla_{\bm{x}} f_y$:
	\begin{equation}
		3 \leqslant 
                \frac{\rho_\text{adv}(f, \bm{x},\epsilon, \ell_\infty)}
                {\rho_\text{rand}(f, \bm{x},\epsilon, \ell_\infty)} 
                \leqslant 3\Vert \nabla_x f_y(\bm{x}) \Vert_0.
	\end{equation}
\end{theorem}
\begin{proof}
	We assume $f$ to be differentiable, where the surrogate gradient is used.
	When $\epsilon$ is small, we can expand $f_y(\bm{x}+\epsilon\cdot\delta)$ at $f(\bm{x})$ by the first-order Taylor expansion
	\begin{equation}
		f_y(\bm{x}+\epsilon\cdot\delta) \approx f_y(\bm{x}) + \epsilon \nabla f_y(\bm{x})^T\delta.
	\end{equation}
	So, $f_y(\bm{x}+\epsilon\cdot\delta) - f_y(\bm{x}) \approx \epsilon \nabla f(\bm{x})^T\delta$.
	As $\delta \in \mathbb{R}^m$ and $\delta_i \sim Unif([-1,1])$, we have
	\begin{equation}
		\mathbb{E} (\delta_i\delta_j) = \left\{
		\begin{aligned}
			0 \quad i\ne j \\
			\frac{1}{3} \quad i=j.
		\end{aligned}
		\right.
	\end{equation}
	Therefore, $\rho_\text{rand}(f,\bm{x},\epsilon,\ell_\infty)$ can be approximated as follows:
	\begin{equation}
		\begin{split}
			\rho_\text{rand}(f,\bm{x},\epsilon,\ell_\infty)
			&= \mathbb{E}_{\delta\sim Unif(cube)} \left( f_y(\bm{x}+\epsilon\cdot\delta) - f_y(\bm{x}) \right)^2 
			\approx \epsilon^2 \nabla f_y (\bm{x})^T \mathbb{E}_{\delta} (\delta \delta^T) \nabla f_y(\bm{x}) 
			= \frac{1}{3} \epsilon^2 \Vert \nabla f(\bm{x}) \Vert_2^2.
		\end{split}
	\end{equation}
	
	On the other hand, 
	\begin{equation}
		\begin{split}
			\rho_\text{adv}(f,\bm{x},\epsilon,\ell_\infty) &= 
			\sup_{\delta\sim Unif(cube)} \left( f_y(x+\epsilon\cdot\delta) - f_y(\bm{x}) \right)^2
			\approx \epsilon^2 \left( \sup_{\delta\sim Unif(cube)} | \nabla f_y(\bm{x})^T \delta | \right)^2\\
			&= \epsilon^2 \left( \nabla f_y(\bm{x})^T \text{sign}(\nabla f_y(\bm{x})) \right)^2
			= \epsilon^2 \Vert \nabla f_y(\bm{x}) \Vert_1^2.
		\end{split}
	\end{equation}
	
	Consequently, the gap between $\rho_\text{adv}(f,\bm{x},\epsilon,\ell_\infty)$ and $\rho_\text{rand}(f,\bm{x},\epsilon,\ell_\infty)$ can be measured by 
	\begin{equation}
		\frac{\rho_\text{adv}(f,\bm{x},\epsilon,\ell_\infty)}{\rho_\text{rand}(f,\bm{x},\epsilon,\ell_\infty)} \approx 3 \frac{\Vert \nabla f_y(\bm{x})\Vert_1^2}{\Vert \nabla f_y(\bm{x})\Vert_2^2},
	\end{equation}
	which can be bounded by
	\begin{equation}
		3 \leq \frac{\rho_\text{adv}(f,\bm{x},\epsilon,\ell_\infty)}{\rho_\text{rand}(f,\bm{x},\epsilon,\ell_\infty)} \leqslant 3\Vert \nabla f_y(\bm{x}) \Vert_0.
	\end{equation}

	[proof of the inequality] For an $m$-dimensional vector $u\in\mathbb{R}^m$, we have $\Vert u \Vert_1 \geqslant \Vert u \Vert_2$.
	Because
	\begin{equation}
		\begin{split}
			\Vert u \Vert_1^2 &= (\sum_{i=1}^m |u_i|)^2 
			= \sum_{i=1}^m u_i^2 + \sum_i\sum_{j\ne i} |u_iu_j| \geqslant \sum_{i=1}^m u_i^2 = \Vert u \Vert_2^2.
		\end{split}
	\end{equation}
	Moreover, let $a\in\mathbb{R}^m$ be an $m$-dimensional vector with $a_i = \text{sign}(u_i)$
	Then 
	\begin{equation}
		\begin{split}
			\Vert u \Vert_1 &= \sum_{i=1}^m |u_i| = \sum_{i=1}^m u_i a_i
			\leqslant \left( \sum_{i=1}^m u_i^2 \right)^\frac{1}{2} \left( \sum_{i=1}^m a_i^2 \right)^\frac{1}{2}~\text{(Cauchy Schwartz inequality)}
			= \Vert u \Vert_2 \sqrt{\Vert u\Vert_0}
		\end{split}
	\end{equation}
	
\end{proof}

\section{Derivation of Equation~(\ref{eq:fy_gradient})}
\label{sec:der_eq10}

Let $\bm{x}$ denote the image, and $\{\bm{x}[1], \bm{x}[2], \dots, \bm{x}[T]\}$ represent the input image series. In our paper, we use $\bm{x}[t]=\bm{x}$ for all $t=1,\dots,T.$ The network is denoted by $f$ and the output of the network $f$ in the last layer is a vector $f^L(\bm{x})\in \mathbb{R}^{N \times 1}$, where $N$ represents the number of classes, i.e.
\begin{equation}
    f^L(\bm{x}) = \left( 
        \sum_{t=1}^T s_1^L[t], \dots, \sum_{t=1}^T s_N^L[t]
    \right).
\end{equation}
Based on Equation~(\ref{eq:f_y}), the component related to the true label $y$ is defined as
\begin{equation}
     f_y(\bm{x}) = \frac{
    e^{f_y^L}
    }{e^{f_y^L} + e^{f_{\tilde{y}^L}}}
     = \frac{
    e^{\sum_{t=1}^T s_y^L(t)}
    }{e^{\sum_{t=1}^T s_y^L(t)} + e^{\sum_{t=1}^T s_{\tilde{y}}^L(t)}},
\end{equation}
According to the chain rule, the gradient of $f_y(\bm{x})$ with respect to the input $\bm{x}$ is
\begin{equation}
     \nabla_{\bm{x}} f_y(\bm{x}) =  \frac{\partial f_y }{\partial f_y^L}
                 \cdot
                \nabla_{\bm{x}} \left( \sum_t s_y^L[t] \right) +
                \frac{\partial f_y }{\partial f_{\tilde{y}}^L}
                 \cdot \nabla_{\bm{x}} \left( \sum_t s_{\tilde{y}}^L[t] \right).
\end{equation}
In this formula, the gradient of $\nabla_{\bm{x}} \left( \sum_t s_i^L[t] \right)~(i=y,\tilde{y})$ in the last layer can be further expresses as
\begin{equation}
    \nabla_{\bm{x}} \left( \sum_t s_i^L[t] \right) =  \sum_t \nabla_{\bm{x}} s_i^L[t] 
    = \sum_t \sum_{\tilde{t}=1}^t \nabla_{\bm{x}[\tilde{t}]} s_i^L[t] .
\end{equation}
Finally, the gradient $\nabla_{\bm{x}} f_y(x,w)$ is written as
\begin{equation}
    \nabla_{\bm{x}} f_y(\bm{x}) = \sum_{i=y,\tilde{y}} \left( \frac{\partial f_y
                    }{
                        \partial f_i^L
                    }
                    \left(
                        \sum_{t=1}^T \sum_{\tilde{t}=1}^t \nabla_{\bm{x}[\tilde{t}]} s_i^L[t]
                    \right) \right).
\end{equation}

\section{Proof of Proposition~\ref{prop:L1approx}}
\label{append:proof_prop}
\renewcommand{\theproposition}{4.4}
\begin{proposition}
    Let $d$ denote the signed input gradient direction: $d=sign(\nabla_{\bm{x}} f_y(\bm{x}))$, and $h$ be the finite difference step size. Then, the $\ell_1$ gradient norm can be approximated as:
    \begin{equation}
        \Vert \nabla_{\bm{x}} f_y(\bm{x}) \Vert_1 \approx \left| \frac{
        f_y(\bm{x}+h\cdot d) - f_y(\bm{x})
        }{
        h
        } \right|
    \end{equation}
\end{proposition}
\begin{proof}
    The first order Taylor estimation of $f_y(\bm{x}+h\cdot d)$ at point $\bm{x}$ is
    \begin{equation}
        f_y(\bm{x}+h\cdot d) \approx f_y(\bm{x}) + h\cdot \nabla_{\bm{x}} f_y^T(\bm{x}) d
         = f_y(\bm{x}) + h\Vert \nabla_{\bm{x}} f_y(\bm{x}) \Vert_1.
    \end{equation}
    Therefore, $\Vert \nabla_{\bm{x}} f_y(\bm{x}) \Vert_1$ can be approximated by
    \begin{equation}
        \left| \frac{
        f_y(\bm{x}+h\cdot d) - f_y(\bm{x})
        }{
        h
        } \right|
    \end{equation}
\end{proof}

\section{Training Settings}
\label{append:train_settings}
We use the same training settings for all architectures and datasets. Our data augmentation techniques include RandomCrop, RandomHorizontalFlip, and zero-mean normalization. During training, we use the CrossEntropy loss function and Stochastic Gradient Descent optimizer with momentum. The learning rate $\eta$ is controlled by the cosine annealing strategy~\cite{loshchilov2016sgdr}. We utilize the Backpropagation Through Time (BPTT) algorithm with a triangle-shaped surrogate function, as introduced by \cite{esser2016convolutional}.
When incorporating sparsity gradient regularization, we set the step size of the finite difference method to 0.01. Also, we use a $\lambda=0.002$ on CIFAR-10/CIFAR10-DVS and $\lambda=0.001$ on CIFAR-100 for SR* method. For vanilla SR, we set $\lambda=0.008$ on CIFAR-10 and $\lambda=0.002$ on CIFAR-100/CIFAR10-DVS. 
The detailed training hyper-parameters are listed below.

\begin{table}[htbp]
\caption{Detailed training setting.}
\label{tab:ablation1}
\small
\centering
\begin{threeparttable}
\begin{tabular}{ccccccccccccc}
\toprule
Timestep & Initial LR & Batchsize & Weight Decay & Epochs & Momentum & $h$  & AT   & $\epsilon$ & PGD-step \\ 
\midrule
\multicolumn{10}{c}{CIFAR-10/100 Dataset}  \\ \midrule
8  & 0.1 & 64 & 5e-4 & 200 & 0.9 & 0.01 & PGD5 & 2/255 & 0.01  \\ \midrule
\multicolumn{10}{c}{CIFAR10-DVS Dataset}  \\ \midrule
10 & 0.1 & 24 & 5e-4 & 200 & 0.9 & 0.01 & FGSM & 2/255 & / \\ \bottomrule
\end{tabular}
\end{threeparttable}
\end{table}

\section{Evaluation Settings}
\label{append:eval_settings}
As mentioned in the main text, we consider an ensemble attack approach for SNNs. This involves utilizing a diverse set of surrogate gradients and considering both STBP-based and RGA-based attacks. We conduct the following attacks as the ensemble attack.
We consider an ensemble attack to be successful for a test sample as long as the model is fooled with any of the attacks from the ensemble.

\begin{itemize}
\setlength{\itemsep}{0pt}
\setlength{\parsep}{0pt} 
\setlength{\parskip}{0pt}
    \item \textbf{STBP-based attack with triangle-shaped surrogate function}, with the hyper-parameter $\gamma=1$~\citep{esser2016convolutional}. 

    \begin{align}
        \frac{\partial s^l_i(t)}{\partial u^l_i(t)} = \frac{1}{\gamma}\left| \lvert \gamma - \left \lvert u^l_i(t) - \theta \right| \rvert \right\rvert.
    \end{align}
    
    \item \textbf{STBP-based attack with sigmoid-shaped surrogate function}, with the hyper-parameter $\gamma$ being 4.
    \begin{align}
    \frac{\partial s^l_i(t)}{\partial u^l_i(t)} = \frac{1}{1+\exp{(-\gamma(u^l_i(t)-\theta))}}
    \end{align}
    
    \item \textbf{STBP-based attack with arc tangent surrogate function}, with the hyper-parameter $\gamma$ set to 2.
    \begin{align}
    \frac{\partial s^l_i(t)}{\partial u^l_i(t)} = \frac{\gamma}{2(1 + (\frac{\pi}{2}\gamma (u^l_i(t)-\theta))^2)}
    \end{align}

    \item \textbf{RGA-based attack with rate-based gradient estimation}, where the setting follows the paper~\citep{bu2023rate}.
\end{itemize}

\section{Coefficient Parameter Search on CIFAR-100}
\label{append:lambda_cifar100}

\begin{figure}[!b]  
    \centering
    \subfloat{
    \includegraphics[width=0.3\textwidth]{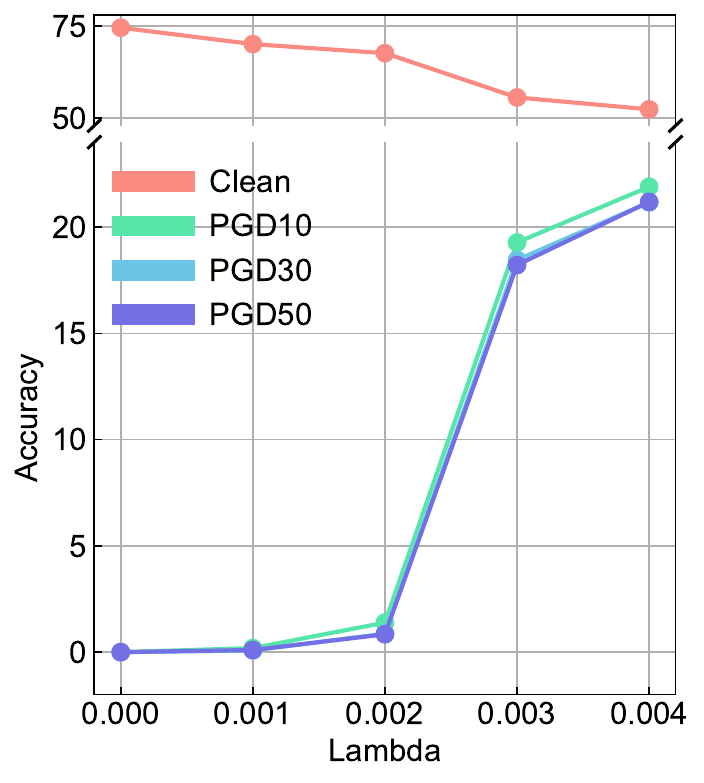}}
    \subfloat{
    \includegraphics[width=0.3\textwidth]{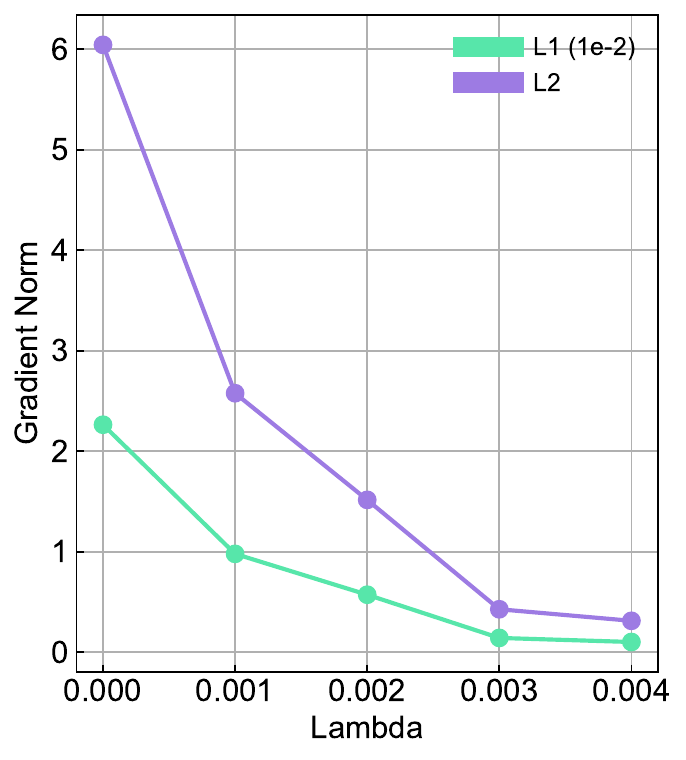}}
    \caption{The influence of the coefficient parameter $\lambda$ on classification accuracy and gradient sparsity. Left: Fluctuations in clean accuracy and adversarial accuracy under PGD attacks. Right: The $\ell_1$ and $\ell_2$ norms of the gradient with varying $\lambda$.}
    \label{fig:lambda_cifar100}
\end{figure}

Figure~\ref{fig:lambda_cifar100} (left) shows the relationship between the coefficient parameter $\lambda$, clean accuracy, and adversarial accuracy on the CIFAR-100 dataset.
To make a trade-off between classification accuracy on clean images and adversarial images, we choose $\lambda=0.002$ for SR-WRN-16 on the CIFAR-100 dataset.
Figure~\ref{fig:lambda_cifar100} (right) illustrates that the $\ell_1$ and $\ell_2$ norm of $\nabla_{\bm{x}} f_y$ decrease significantly with the increase of $\lambda$.

\section{Extensive Evaluations With Varying Widths of Surrogate Functions}
\label{append:varying_width}

For the ensemble attack to be meaningful and reveal any impact of gradient obfuscation, we run an evaluation with different widths $\gamma$ in the surrogate function. Specifically, we apply the attack with $\gamma\in [0.1,3.0]$ in fine-grained steps of 0.1, and report the results of the PGD10 attack on VGG-11 models with different training algorithms on the CIFAR-10 dataset in Table~\ref{tab:surrogate_width}.

\begin{table}[htbp]
\centering
\caption{The classification accuracy (\%) under the ensemble attack with different $\gamma$.}
\label{tab:surrogate_width}
\begin{tabular}{lccc}
\toprule
Attacks            & RAT  & AT & SR* \\
\midrule
PGD10 (w/o ensemble)          & 16.16        & 21.32       & 33.67        \\
PGD10 (w/ ensemble)           & 11.53        & 18.18       & 30.54        \\
PGD10 ($\gamma\in [0.1,3.0]$)  & 11.87        & 16.16       & 27.06       \\
\bottomrule
\end{tabular}
\end{table}

We compare three different attack combinations to evaluate the impact of different surrogate functions on the attack strength. We select RAT, PGD5-AT, and SR*(PGD5+SR) models as the target models. For PGD10(w/o ensemble), we only use the Triangle-shaped surrogate function, which is identical to the one used in training. For PGD10 (w/ ensemble), we use the attack combination as described in Section~\ref{sec:experiment}. For PGD10(
$\gamma\in [0.1, 3.0]$), we incorporate 30 different Triangle-shaped surrogate functions with widths ranging from $[0.1, 30]$.

We find that both ensemble attack methods significantly improve attack performance and mitigate the impact of gradient obfuscation. This indicates that both the shape and width of the surrogate function can influence the capability of the adversary. Although the PGD10($\gamma\in [0.1, 3.0]$) attack is slightly more effective than the ensemble attack used in Section~\ref{sec:experiment}, it uses a 30-fold fine-grained grid search over attack hyperparameters for each image, which is considerably more computationally expensive.

In conclusion, we demonstrate that changing the width of surrogate functions does not significantly influence the capability of the adversary any better than using different surrogate gradient shapes. Additionally, the proposed SR* strategy exhibits improved robustness under both scenarios.

\section{Comparison in Computational Cost}
\label{append:cost}

The computational costs for one epoch training of various algorithms, including vanilla, PGD5 AT, RAT, SR, and SR*(PGD5+SR), on the CIFAR-10 dataset using the VGG11 architecture are summarized in Table~\ref{tab:computational_cost}. From the table, we observe that a single SR incurs a computational cost 1.5 times that of RAT but consumes less than half the time needed by PGD5 AT. The computational cost of SR* is the highest among all training algorithms since it combines both SR and AT. However, models trained with SR* achieve the best robustness compared to models trained with other methods. 

\begin{table}[htbp]
    \centering
    \caption{The computational cost in one epoch of different training algorithms.}
    \begin{tabular}{cccccc}
    \toprule
        Vanilla & PGD5-AT & RAT & SR & SR*  \\ \midrule
        65s & 459s & 134s & 193s & 583s  \\ \bottomrule
    \end{tabular}
    \label{tab:computational_cost}
\end{table}

\end{document}